\documentclass[a4paper]{article}

\usepackage[english]{babel}
\usepackage[utf8x]{inputenc}
\usepackage[T1]{fontenc} 
\usepackage{scalerel}

\usepackage[a4paper,top=3cm,bottom=2cm,left=3cm,right=3cm,marginparwidth=1.75cm]{geometry}

\usepackage{graphicx}
\usepackage{subcaption}
\usepackage{amsmath}
\usepackage{mathtools}
\usepackage{amssymb}
\usepackage{amsthm}
\usepackage{natbib}
\usepackage{comment}
\usepackage{enumitem} 
\usepackage{subcaption}
\usepackage{cases}
\usepackage{graphicx}
\usepackage{wrapfig}
\usepackage[colorinlistoftodos]{todonotes}
\usepackage[colorlinks=true, allcolors=blue]{hyperref}
\usepackage{empheq} 
\usepackage[most]{tcolorbox} 
\usepackage[linesnumbered,ruled]{algorithm2e}
\usepackage{algorithmic}
\SetKwRepeat{Do}{do}{while}%
\usepackage{multirow}
\usepackage{bm}
\usepackage{siunitx}
\usepackage{footnote}
\makesavenoteenv{tabular}
\makesavenoteenv{table}
\usepackage{breakcites}

\usepackage{xcolor}
\usepackage{comment}
\usepackage[normalem]{ulem} 
\newcommand\str{\bgroup\markoverwith
{\textcolor{red}{\rule[0.5ex]{2pt}{1.5pt}}}\ULon} 
\usepackage{pifont}
\newcommand{\al}{\alpha}
\newcommand{\be}{\beta}

\newcommand{\eps}{\epsilon}
\newcommand{\tsigma}{\tilde{\sigma}}

\newcommand{\bigO}{\mathcal{O}}

\newcommand{\R}{{\mathbb{R}}}

\newcommand{\T}{\mathcal{T}}
\newcommand{\F}{\mathcal{F}}
\newcommand{\I}{\mathcal{I}}
\newcommand{\B}{\mathcal{B}}
\newcommand{\D}{\mathcal{D}}
\newcommand{\E}{\mathbb{E}}

\newtheorem{theorem}{Theorem}[section]  
\newtheorem{definition}[theorem]{Definition}

\newtheorem{lemma}[theorem]{Lemma}
\newtheorem{remark}[theorem]{Remark}
\newtheorem{corollary}[theorem]{Corollary}

\newtheorem{assumption}[theorem]{Assumption}

\newcommand{\beq}{\begin{equation}}
\newcommand{\eeq}{\end{equation}}
\newcommand{\beqa}{\begin{eqnarray}}
\newcommand{\eeqa}{\end{eqnarray}}
\newcommand{\beqs}{\begin{equation*}}
\newcommand{\eeqs}{\end{equation*}}
\newcommand{\beqas}{\begin{eqnarray*}}
\newcommand{\eeqas}{\end{eqnarray*}}
\DeclarePairedDelimiter\ceil{\lceil}{\rceil}

\newtcbox{\mymath}[1][]{%
    nobeforeafter, math upper, tcbox raise base,
    enhanced, colframe=blue!30!black,
    colback=blue!30, boxrule=1pt,
    #1}
    
\usepackage{caption}
\usepackage{graphicx}
\usepackage{subcaption}

\begin{document}

\title{On the Convergence Theory of Gradient-Based Model-Agnostic Meta-Learning Algorithms\thanks{To appear in the proceedings of the $23^{rd}$ International Conference on Artificial Intelligence and Statistics (AISTATS) 2020.}}
\author{Alireza Fallah\thanks{Department of Electrical Engineering and Computer Science, Massachusetts Institute of Technology, Cambridge, MA, USA. \{afallah@mit.edu, asuman@mit.edu\}.} , Aryan Mokhtari\thanks{Department of Electrical and Computer Engineering, The University of Texas at Austin, Austin, TX, USA. mokhtari@austin.utexas.edu.} , Asuman Ozdaglar$^\dag$}
\date{}
\maketitle

\begin{abstract}
We study the convergence of a class of gradient-based Model-Agnostic Meta-Learning (MAML) methods and characterize their overall complexity as well as their best achievable accuracy in terms of gradient norm for \textit{nonconvex} loss functions. We start with the MAML method and its first-order approximation (FO-MAML) and highlight the challenges that emerge in their analysis. By overcoming these challenges not only we provide the first theoretical guarantees for MAML and FO-MAML in nonconvex settings, but also we answer some of the unanswered questions for the implementation of these algorithms including how to choose their learning rate and the batch size for both tasks and  datasets corresponding to tasks. In particular, we show that MAML can find an $\eps$-first-order stationary point ($\eps$-FOSP) for any positive $\eps$ after at most $\mathcal{O}(1/\eps^2)$ iterations at the expense of requiring second-order information. We also show that FO-MAML which ignores the second-order information required in the update of MAML cannot achieve any small desired level of accuracy, i.e., FO-MAML cannot find an $\eps$-FOSP for \textit{any} $\eps>0$. We  further propose a new variant of the MAML algorithm called Hessian-free MAML which preserves all theoretical guarantees of MAML, without requiring access to second-order information.
\end{abstract}


%

\section{Introduction}
In several artificial intelligence problems, ranging from robotics to image classification and pattern recognition, the goal is to design systems that use prior experience and knowledge to learn new skills more efficiently. \textit{Meta-learning} or \textit{learning to learn} formalizes this goal by using data from previous tasks to learn update rules or model parameters that can be fine-tuned to perform well on new tasks with small amount of data \citep{LearningToLearn}. Recent works have integrated this paradigm with neural networks including learning the initial weights of a neural network \citep{finn17a, Reptile}, updating its architecture \citep{baker2016designing, zoph2016neural, zoph2018learning}, or learning the parameters of optimization algorithms using recurrent neural networks \citep{ravi2016optimization, NIPS2016_Andry}.

A particularly effective approach, proposed in \citep{finn17a}, is the gradient-based meta-learning in which the parameters of the model are explicitly trained such that a small number of gradient steps with a small amount of training data from a new task will produce good generalization performance on that task. This method is referred to as \textit{model-agnostic meta learning (MAML)} since it can be applied to any learning problem that is trained with gradient descent. Several papers have studied the empirical performance of MAML for nonconvex settings \citep{Reptile, MAML++, Meta-SGD, CAVIA, grant2018recasting, alpha-MAML, al-shedivat2018continuous}. However, to the best of our knowledge, its convergence properties have not been established for general non-convex functions.

In this paper, we study the convergence of variants of MAML methods for nonconvex loss functions and establish their computational complexity as well as their best achievable level of accuracy in terms of gradient norm.    
More formally, let $\T = \{\T_i\}_{i \in \I}$ denote the set of all tasks and let $p$ be the probability distribution over tasks $\T$, i.e., task $\T_i$ is drawn with probability $p_i= p(\T_i)$. We represent the loss function corresponding to task $\T_i$ by $f_i(w): \R^d \to \R$ which is parameterized by the same $w \in \R^d$ for all tasks. Here, the loss function $f_i$ measures how well an action $w$ performs on task $\T_i$. The goal of expected risk minimization is to minimize the expected loss over all tasks, i.e., 
\begin{equation}\label{f_def}
\min f(w) := \E_{i\sim p}[f_i(w)].
\end{equation}
In most learning applications, the loss function $f_i$ corresponding to task $\T_i$ is defined as an expected loss with respect to the probability distribution which generates data for task $\T_i$, i.e., 
$f_i(w) := \E_{\theta}[f_i(w, \theta)]$.
In this case, the gradient and Hessian of $f_i$ can be approximated by $\nabla f_i(w,\D) := \frac{1}{|\D|} \sum_{\theta \in \D} \nabla f_i (w,\theta)$ and $\nabla^2 f_i(w,\D) := \frac{1}{|\D|} \sum_{\theta \in \D} \nabla^2 f_i (w,\theta)$, respectively, where $\D$ is a batch chosen from the dataset of task $\T_i$.  

In traditional statistical learning, we solve Problem \eqref{f_def} as we expect its solution to be a proper approximation for the optimal solution of a new unseen task $\T_i$. However, in model-agnostic meta-learning, we aim to find the best point that performs well as an initial point for learning a new task $\T_i$ when \textit{we have budget for running a few steps of gradient descent} \citep{finn17a}. For simplicity, we focus on finding an initialization $w$ such that, after observing a new task $\T_i$, one gradient step would lead to a good approximation for the minimizer of $f_i(w)$. We can formulate this goal as 
\begin{equation}\label{main_prob}
\min F(w) := \E_{i \sim p}\left[ F_i(w)\right]
:= \E_{i \sim p}\left[f_i(w - \al \nabla{f_i(w)})\right],
\end{equation}
where $\alpha>0$ is the stepsize for the update of gradient descent method and $F_i(w)$ denotes $f_i(w - \al \nabla{f_i(w)})$.

Problem \eqref{main_prob} is defined in a way that its optimal solution would perform well in expectation when we observe a task and look at the output after running a single step of \textit{gradient} descent.\footnote{We only consider the case that one step of gradient is performed for a new task, but, indeed, a more general case is when we perform multiple steps of gradient descent (GD). However, running more steps of GD comes at the cost of computing {multiple Hessians} and for simplicity of our analysis we only focus on a single iteration of GD.} 

\begin{wrapfigure}{r}{0.4\textwidth}
\vspace{-0.5cm}
\hspace{-2.2cm}  \includegraphics[width=0.75\textwidth]{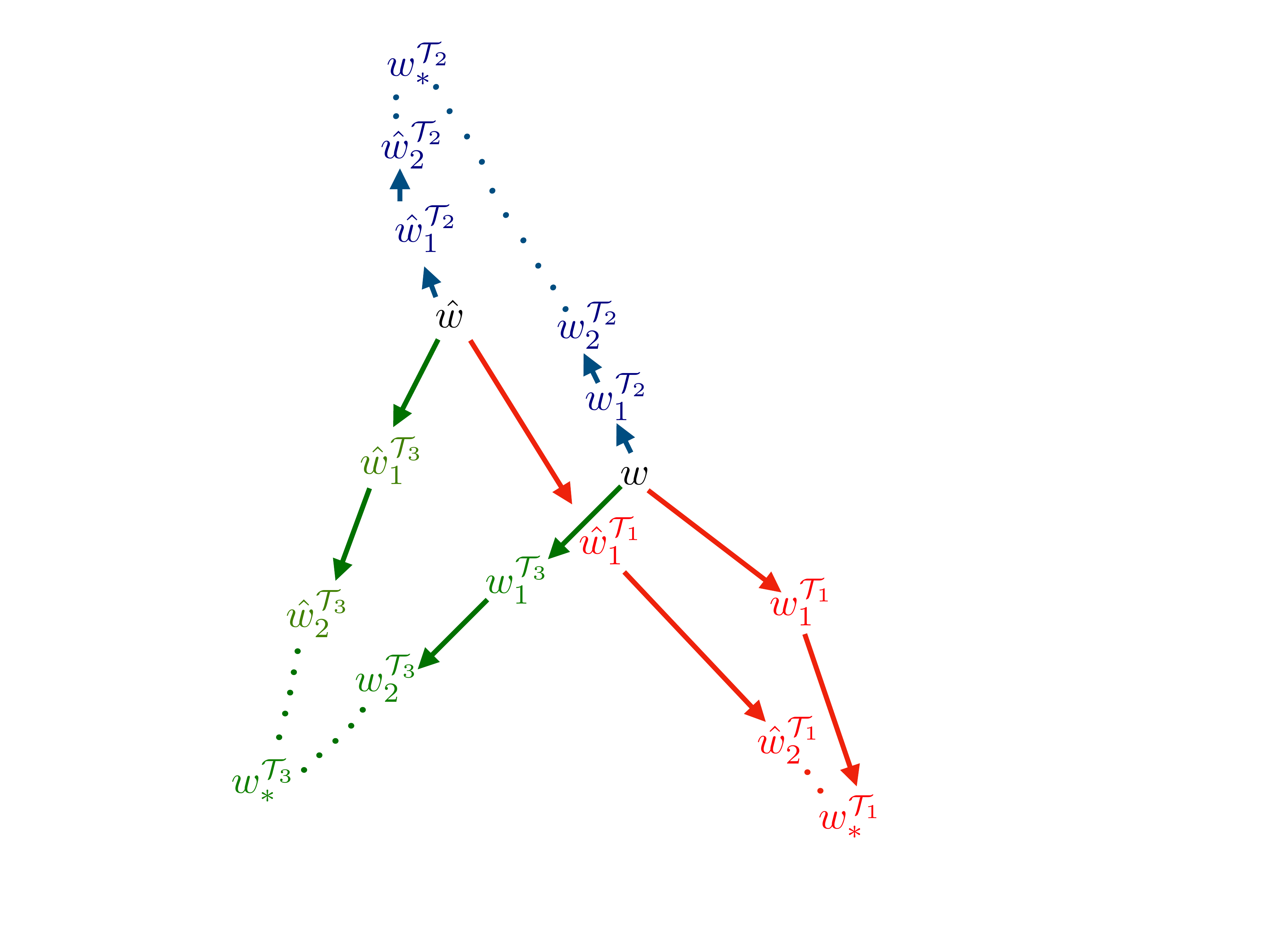}
\vspace{-1.5cm}
    \caption{\label{Fig1_MAML} Comparison of the performance of the optimal solution of the statistical learning problem in \eqref{f_def} and the optimal solution of the met-learning problem in \eqref{main_prob} when we have budget for two steps of gradient descent update.}
    \vspace{-0.2cm}
\end{wrapfigure}

However, in most applications, computing the exact gradient for each task is costly and we can only run steps of the stochastic gradient descent (SGD) method. In this case, our goal is to find a point $w$ such that when a task $\T_i$ is chosen, after running one step of SGD, the resulting solution performs well in expectation. In particular, we assume we have access to the stochastic gradient $\tilde{\nabla} f_i(w,\D^i_{test})$ which is an unbiased estimator of $\nabla f_{i}(w)$ evaluated using the batch $\D^i_{test}$ with size $D_{test}$. In this formulation, our goal would change to solving the problem
\begin{equation}\label{main_prob_2}
\min \hat{F}(w)
:=\E_{i \sim p}\left[\E_{\D^{i}_{test}} \left[ f_i(w - \al \tilde{\nabla} f_i(w,\D^i_{test}))\right] \right],
\end{equation}
where the expectation is taken with respect to selection of task $i$ as well as selection of random set $\D^i_{test}$ for computing stochastic gradient. Throughout the paper, we will clarify the connection between $F$ and $\hat{F}$, and we report our results for both of these functions.
 
It is worth emphasizing that the solution of the standard expected risk minimization in \eqref{f_def} gives us the best answer when we are given many tasks and we plan to choose only ``one action'' that  performs well in expectation, when we observe a new unseen task. On the other hand, the solution of the expected risk minimization in \eqref{main_prob} is designed for the case that we have access to a large number of tasks and we aim to choose ``an action that after one or more steps of gradient descent'' performs well for an unseen task. In the first case, we naturally choose an action that is closer to the optimal solutions of the tasks that have higher probability, but in the second case we choose an action that is closer to the optimal solutions of the tasks that have higher probability and are harder for gradient descent to solve them. For instance, when the loss functions are strongly convex and smooth, a harder task (minimization problem) for gradient descent is the problem that has a larger condition number. Therefore, the solution of \eqref{main_prob} is naturally closer to the solution of those tasks that have larger condition numbers. 

To better highlight the difference between the solutions of the statistical learning problem in~\eqref{f_def} and the meta-learning problem in~\eqref{main_prob}, we consider an example where we have access to three equally likely tasks $\T_1$, $\T_2$, and $\T_3$ with the optimal solutions $w_{*}^{\T_1}$, $w_{*}^{\T_2}$, $w_{*}^{\T_3}$, respectively; see Figure \ref{Fig1_MAML}. Here, $w$ is the solution of Problem~\eqref{f_def} and $\hat{w}$ is the solution of Problem~\eqref{main_prob}. 
In this example, task $\T_1$ is the easiest task as we can make a lot of progress with only two steps of GD and  task $\T_2$ is the hardest task as we approach the optimal solution slowly by taking gradient steps. As we observe in Figure~\ref{Fig1_MAML}, for task $\T_3$, if we start from $w$ the outcome after running two steps of GD is almost the same as starting from $\hat{w}$. For task $\T_1$, however, $w$ is a better initial point compared to $\hat{w}$, but the error of their resulting solution after two steps of GD are not significantly different. This is due to the fact that $\T_1$ is easy and for both cases we get very close to the optimal solution even after two steps of GD.  The difference between starting from $w$ and $\hat{w}$ is substantial when we aim to solve task $\T_2$ which is the hardest task. Because of this difference, the updated variable after running two steps of GD has a lower expected error when we start from $\hat{w}$ comparing to the case that we start from $w$. This simple example illustrates the fact that if we know a-priori that after choosing an model we are allowed to run a single (or more) iteration of GD to learn a new task, then it is better to start from the minimizer of \eqref{main_prob} rather than the minimizer of~\eqref{f_def}.


\begin{table*}[t!]
\footnotesize
\begin{center}
\begin{tabular}{ |l|c|c|c|c|c|}
\hline
\multirow{3}{*}{\textbf{Algorithm}}&  \multicolumn{4}{c|}{\textbf{Having access to sufficient samples}} & {\textbf{K-shot Learning}} \\

 & Best accuracy & Iteration & \# samples/ & Runtime/ & Best accuracy\\ 
 & possible & complexity & iteration & iteration & possible\\
\hline
\multirow{2}{*}{\textbf{MAML}} & \multirow{2}{*}{$\| \nabla F(w) \| \leq \eps$} & \multirow{2}{*}{$\bigO(1/\eps^2)$} & \multirow{2}{*}{$\bigO(1/\eps^4)$} & \multirow{2}{*}{$\bm{\bigO(d^2)}$} & \multirow{2}{*}{$\| \nabla F(w) \| \leq \bigO(\tsigma/\sqrt{K})$} \\
 & & & & &\\
\hline
\multirow{2}{*}{\textbf{FO-MAML}} & \multirow{2}{*}{$\| \nabla F(w) \| \leq \bm{\bigO(\alpha \sigma)}$} & \multirow{2}{*}{$\bigO(1/(\alpha^2 \sigma^2))$} & \multirow{2}{*}{$\bigO(1/(\alpha^4 \sigma^2))$} & \multirow{2}{*}{$\bigO(d)$} & \multirow{2}{*}{$\| \nabla F(w) \| \leq \bigO(\bm{\sigma}+ \tsigma/\sqrt{K})$} \\
 & & & & &\\
\hline
\multirow{2}{*}{\textbf{HF-MAML} } & \multirow{2}{*}{$\| \nabla F(w) \| \leq \eps$} & \multirow{2}{*}{$\bigO(1/\eps^2)$} & \multirow{2}{*}{$\bigO(1/\eps^4)$} & \multirow{2}{*}{${\bigO(d)}$} & \multirow{2}{*}{$\| \nabla F(w) \| \leq \bigO(\tsigma/\sqrt{K})$} \\
& & & & &\\
\hline
\end{tabular}
\end{center}
\caption{\label{table-1}Our theoretical results for convergence of MAML, first-order approximation of MAML (FO-MAML), and our proposed Hessian-free MAML (HF-MAML) to a first-order stationary point (FOSP) in nonconvex settings. Here, $d$ is the problem dimension, $\sigma$ is a bound on the standard deviation of $\nabla f_i(w)$ from its mean $\nabla f(w)$, and $\tsigma$ is a bound on the standard deviation of $\nabla f_i(w. \theta)$, an unbiased estimate of $\nabla f_i(w)$, from its mean $\nabla f_i(w)$, for every $i$. For any $\eps>0$, MAML can find an $\eps$-FOSP, while each iteration has a complexity of $\mathcal{O}(d^2)$. FO-MAML has a lower complexity of $\mathcal{O}(d)$, but it cannot reach a point with gradient norm less than $\mathcal{O}(\alpha \sigma)$. HF-MAML has the best of both worlds, i.e., HF-MAML has a cost of $\mathcal{O}(d)$, and it can find an $\eps$-FOSP for any $\eps>0$.}
\end{table*}%

\subsection{Our contributions}

 In this paper, we provide the first theoretical guarantees for the convergence of MAML algorithms to first order stationarity for \textit{non-convex} functions. We build our analysis upon interpreting MAML as a SGD method that solves Problem \eqref{main_prob} while we show the analysis of MAML is significantly more challenging due to various reasons, including unbounded smoothness parameter and the biased estimator of gradient used in the update rule of MAML. 
Overcoming these challenges, we characterize the iteration and sample complexity of MAML method and shed light on the relation of batch sizes and parameters of MAML with its convergence rate and accuracy. Using these results, we provide an explicit approach for tuning the hyper-parameters of MAML and also the required amount of data to reach a first-order stationary point of \eqref{main_prob}. A summary of the results\footnote{We assume $\sigma$ and $\tsigma$ are small for the results in this section. The general result can be found in Section \ref{sec:Theory}.} and the specific case of $K$-shot learning, where for each task in the inner loop we have access to $K$ samples, is provided in Table \ref{table-1}. Note that in these results, $\sigma$ is a bound on the standard deviation of $\nabla f_i(w)$ from its mean $\nabla f(w)$, and $\tsigma$ is a bound on the standard deviation of $\nabla f_i(w. \theta)$ from its mean $\nabla f_i(w)$, for every $i$. For formal definitions of $\sigma$ and $\tsigma$ please check Assumptions \ref{asm_bounded_var} and \ref{asm_bounded_var_i}, respectively.


As described in \citep{finn17a}, the implementation of MAML is costly\footnote{The cost per iteration is $\bigO(d^2)$ in general. However, it is worth noting that  this cost reduces to $\bigO(d)$ for the case of neural network classifiers using back propagation.} as it requires Hessian-vector product computation. To resolve this issue, \cite{finn17a} suggest ignoring the second-order term in the update of MAML and show that the first-order approximation does not affect the performance of MAML in practice. In our work, we formally characterize the convergence results for this first-order approximation of MAML (\textit{FO-MAML}) and show that if the learning rate $\alpha$ used for updating each task is small or the tasks are statistically close to each other, then the error induced by the first-order approximation is negligible (see Table \ref{table-1}). 
Nevertheless, in general, in contrast to MAML which can find an $\eps$-first order stationary point for any $\eps>0$, FO-MAML is limited to $\eps \geq \bigO(\alpha \sigma)$.

 To address this issue, we introduce a new method, \textit{Hessian-free MAML (HF-MAML)}, which recovers the complexity bounds of MAML without access to second-order information and has a computational complexity of $\mathcal{O}(d)$ per iteration (see Table \ref{table-1}). In fact, we show that,
  for any positive $\eps$, HF-MAML finds an $\epsilon$-FOSP
  while keeping the computational cost $\bigO(d)$ at each iteration. {Hence, HF-MAML has the best of both worlds: it has the low computational complexity of FO-MAML and it achieves any arbitrary accuracy for first-order stationarity as in MAML.} 
  

\section{Related Work}\label{sec:RelatedWorks}

The problem of learning from prior experiences to learn new tasks efficiently has been formulated in various ways. One of the main approaches is designing algorithms for updating the parameters of the optimization methods used for training models \citep{bengio1990learning, bengio1992optimization}.
Recently, many papers have followed this approach \citep{bengio2012, NIPS2011_4443, li2016learning, ravi2016optimization} (see Table 1 in \citep{metz2018learning} for a summary of different approaches and also \citep{MetaLeatningSurvey} for a detailed survey). 
In one of the first theoretical formalizations, \cite{baxter2000model} 
study the problem of \textit{bias learning} where the goal is to find an automatic way for choosing the inductive bias in learning problems. Also, \cite{Franceschi_ICML18} provide a framework for tuning the hyper-parameters of learning algorithms, such as the initialization or the regularization parameter. 

In this paper, we focus on the theoretical analysis of gradient-based model-agnostic meta-learning methods. This setting was first introduced by \cite{finn17a}, and was followed {by several works proposing various algorithms \citep{Reptile, MAML++, Meta-SGD, CAVIA, alpha-MAML, grant2018recasting}. In particular,  \cite{grant2018recasting} introduce an adaptation of MAML for learning the parameters of a prior distribution in a hierarchical Bayesian model. \cite{Meta-SGD}, introduce a variant of MAML that replaces the inner loop learning rate with a vector (which is multiplied with the gradient in an element-wise manner) and then learns this step-vector as well. In another recent work, \cite{Reptile} introduce a new method named \textit{Reptile}, which samples a batch of tasks, and for each task, runs a few steps of stochastic gradient descent on its corresponding loss function. Reptile, then, takes the average from these resulting points and defines the next iterate to be the convex combination of this average and the current iterate.
 However, none of these works provide convergence guarantees for these MAML-type methods which is the main contribution of our paper.} 
\cite{finn19a} study MAML and its extension to online setting for strongly convex functions.
In a recent independent work, \cite{rajeswaran2019metalearning} propose iMAML which implements an approximation of one step of proximal point method in the inner loop. They show when the regularized inner loop loss function is strongly convex, iMAML converges to a first-order stationary point with exact gradient information (no stochasticity due to approximation by a batch of data) and under bounded gradient assumption. These assumptions remove the difficulties associated with unbounded smoothness parameter and biased gradient estimation featured in our analysis (Section~\ref{sec:Theory}).

The online version of meta learning
has also gained attention. In particular, \cite{balcan_ICML} consider the case where the agent sees a sequence of tasks, where each task is an online learning setup where the agent chooses its actions sequentially and suffer losses accordingly. For each task, the agent starts with an initial model, and runs a within task online algorithm to reduce the regret. Finally, the agent updates the initialization or regularization parameter using a meta-update online algorithm. The authors 
study this setting for convex functions and propose a framework using tools from online convex optimization literature. 
A closely related problem to this setting is the lifelong learning framework \citep{balcan2015efficient, pmlr-v54-alquier17a}. 
As an example, in \citep{pmlr-v54-alquier17a}, the authors consider the case were a series of tasks are presented sequentially where each task itself is associated with a dataset which is also revealed sequentially and processed by a within-task method. The authors focus on convex loss functions and introduce a meta-algorithm which updates a prior distribution on the set of feature maps, and use it to transfer information from the observed tasks to a new one.
In a similar line of work, \cite{Denevi_NeurIPS18,Denevi_ICML19} propose an algorithm which incrementally updates the bias regularization parameter using a sequence of observed tasks. 

Also, \cite{finn19a} consider the model-agnostic setting and propose follow the meta leader that achieves a sublinear regret.

    

\section{MAML Algorithm}\label{sec:MAML}

The MAML algorithm was proposed in \citep{finn17a} for solving the stochastic optimization problem in \eqref{main_prob}. In MAML, at each step $k$, we choose a subset  $\mathcal{B}_k$  of the tasks, with each task drawn independently from distribution $p$. For simplicity assume that the size of $\mathcal{B}_k$ is fixed and equal to $B$. Then, the update of MAML is implemented at two levels: (i) inner step and (ii) outer step (meta-step).

\setlength{\textfloatsep}{4pt}
\begin{algorithm}[t!]
    \SetKwInOut{Input}{Input}
    \SetKwInOut{Output}{Output}
	\While{not done}{
    Choose a batch of \textit{i.i.d.} tasks $\B_k \subseteq \I$ with distribution $p$ and with size $B = |\B_k|$;\\
    \For{all $\T_i$ with $i \in \B_k$}{
    Compute $\tilde{\nabla}f_i(w_k,\D_{in}^i)$  using dataset  $\D_{in}^i$;\\
    Set $w_{k+1}^i = w_k - \al \tilde{\nabla}f_i(w_k,\D_{in}^i)$;\\
    }
    Compute $w_{k+1}$ according to the update \eqref{hassan};
    \\
    $k \gets k + 1$;
    }
    \caption{\label{Algorithm1} MAML Algorithm}
\end{algorithm}

 In the inner step, for each task $\T_i$, we use a subset of the dataset $\D_{in}^i$ corresponding to task $\T_i$ to compute the stochastic gradient $\tilde{\nabla}f_i(w_k, \D_{in}^i)$ which is an an unbiased estimator of the gradient ${\nabla}f_i(w_k)$. The stochastic gradient $\tilde{\nabla}f_i(w_k, \D_{in}^i)$ is then used to compute a model $w_{k+1}^i$ corresponding to each task $\T_i$ by a single iteration of stochastic gradient descent, i.e.,
\begin{equation}\label{inner_step}
w_{k+1}^i = w_k - \al \tilde{\nabla}f_i(w_k,D_{in}^i).
\end{equation}
To simplify the notation, we assume the size of dataset $\D_{in}^i$ for all tasks $i$ are equal to $D_{in}$.

In the outer loop, once we have the updated models $\{w_{k+1}^i\}_{i=1}^{B}$ for all tasks in $\B_k$, we compute the revised meta-model $w_{k+1}$ by performing the update 
\begin{align}\label{hassan}
& w_{k+1} =  w_k - \be_k \frac{1}{B} \sum_{i \in \B_k}  \left(I-\al \tilde{\nabla}^2 f_i(w_{k},\D_{h}^i)\right) \tilde{\nabla} f_i(w_{k+1}^i ,\D_{o}^i),
\end{align}

where the stochastic gradient $\tilde{\nabla} f_i(w_{k+1}^i ,\D_{o}^i) $ corresponding to task $\T_i$ is evaluated using the data set $\D_{o}^i$ and the models $\{w_{k+1}^i\}_{i=1}^{B}$ computed in the inner loop, and the stochastic Hessian $ \tilde{\nabla}^2 f_i(w_{k},\D_{h}^i)$ for each  task $\T_i$ is computed using the data set $\D_{h}^i$. Note that the data sets $\D_{in}^i$ used for the inner update are different from the data sets $\D_{o}^i$ and $\D_{h}^i$ used for the outer update. It is also possible to assume that $\D_{o}^i=\D_{h}^i$, but in this paper we assume that $\D_{o}^i$ and $\D_{h}^i$ are independent from each other that allows us to use a smaller batch for the stochastic Hessian computation which is more costly. Here also we assume that the sizes of $\D_o^i$ and $\D_h^i$ are fixed and equal to $D_o$ and $D_h$, respectively. The steps of MAML are outlined in Algorithm~\ref{Algorithm1}.

\vspace{2mm}
\noindent{\textbf{MAML as an approximation of SGD.}}
To better highlight the fact that MAML runs SGD over $F$ in \eqref{main_prob}, consider the update of GD for minimizing the objective function $F$ with step size $\beta_k$ which can be written as
\begin{align}\label{GD_update}
& w_{k+1} = w_k - \be_k \nabla F(w_k) = w_k - \be_k \E_{i \sim p} \left [  \left(I-\al {\nabla}^2 f_i(w_{k})\right)\ \! {\nabla} f_i\!\left(w_k-\al {\nabla} f_i(w_k)\right) \right ] 
\end{align}
 As the underlying probability distribution of tasks $p$ is unknown, evaluation of the expectation in the right hand side of \eqref{GD_update} is often computationally prohibitive. Therefore, one can use SGD for minimizing the function $F$ with a batch $\B_k$ which contains $B$ tasks that are independently drawn. Then, the update is
\begin{align}\label{SGD_update}
& w_{k+1} =  w_k - \frac{\be_k}{B} \sum_{i \in \B_k}  \left(I-\al {\nabla}^2 f_i(w_{k})\right)\ \! {\nabla} f_i\!\left(w_k-\al {\nabla} f_i(w_k)\right).
\end{align}

If we simply replace ${\nabla} f_i$ and ${\nabla}^2 f_i$ with their stochastic approximations over a batch of data points we obtain the update of MAML in~\eqref{hassan}.

\vspace{2mm}

\noindent{\textbf{Smaller batch selection for Hessian approximation.}} The use of first-order methods for solving problem~\eqref{main_prob} requires computing the gradient of $F$ which needs evaluating the Hessian of the loss $f_i$. Indeed, computation of the Hessians $\nabla^2 f_i$ for all the chosen tasks at each iteration is costly. 
One approach to lower this cost is to reduce the batch size $D_h$ used for Hessian approximation. Later in our analysis, we show that one can perform the update in \eqref{hassan} and have an exactly convergent method, while setting the batch size $D_h$ significantly smaller than batch sizes $D_{in}$ and $B$.
\begin{algorithm}[t!]
    \SetKwInOut{Input}{Input}
    \SetKwInOut{Output}{Output}
	\While{not done}{
    Choose a batch of \textit{i.i.d.} tasks $\B_k \subseteq \I$ with distribution $p$ and with size $B = |\B_k|$;\\
    \For{all $\T_i$ with $i \in \B_k$}{
    Compute $\tilde{\nabla}f_i(w_k,\D_{in}^i)$  using dataset  $\D_{in}^i$;\\
    Set $w_{k+1}^i = w_k - \al \tilde{\nabla}f_i(w_k,\D_{in}^i)$;\\
    }
    ${w_{k+1} \gets w_k -  ({\be_k}/{B}) \sum_{i \in \B_k}\tilde{\nabla} f_i(w_{k+1}^i ,\D_{o}^i) }$; 
    \\
    $k \gets k + 1$;
    }
    \caption{\label{Algorithm2} First-Order MAML (FO-MAML)}
\end{algorithm}

\vspace{2mm}

\noindent{\textbf{First-order MAML (FO-MAML).}} To reduce the cost of implementing the update of MAML one might suggest ignoring the second-order term that appears in the update of MAML. In this approach, which is also known as first-order MAML (FO-MAML)  \citep{finn17a}, we update $w_k$ by following the update 
\begin{equation}\label{FO_MAML}
w_{k+1} = w_k - \be_k \frac{1}{B} \sum_{i \in \B_k}  \tilde{\nabla} f_i(w_{k+1}^i ,\D_{o}^i),
\end{equation}
where the points $w_{k+1}^i$ are evaluated based on \eqref{inner_step}. Indeed, this approximation reduces the computational complexity of implementing MAML, but it comes at the cost of inducing an extra error in computation of the stochastic gradient of $F$. We formally characterize this error in our  theoretical results and show under what conditions the error induced by ignoring the second-order term does not impact its convergence. The steps of FO-MAML are outlined in Algorithm~\ref{Algorithm2}.


\section{Hessian-free MAML (HF-MAML)}

To reduce the cost of implementing MAML we propose an approximate variant of MAML  that is \textit{Hessian-free}, i.e., only requires evaluation of gradients, and has a computational cost of $\mathcal{O}(d)$. 
The idea behind our method is that for any function $\phi$, the product of Hessian $\nabla^2 \phi (w) $ by a vector $v$ can be approximated by 
\begin{equation}
    \nabla^2 \phi (w)  v \approx \left[ 
    \frac{\nabla \phi(w+\delta v)-\nabla \phi(w-\delta v)}{2\delta}\right]
\end{equation}
with an error of at most $\rho \delta \|v\|^2$, where $\rho $ is the parameter for Lipschitz continuity of the Hessian of $\phi$. Based on this approximation,
we propose a computationally efficient approach for 
minimizing the expected loss $F$ defined in \eqref{main_prob}
which we refer to it as Hessian-free MAML (HF-MAML). As the name suggests the HF-MAML is an approximation of the MAML that does not require evaluation of any Hessian, while it provides an accurate approximation of MAML. To be more precise, the update of HF-MAML is defined as 
\begin{align}\label{HF_MAML_main_update}
w_{k+1} =  w_k - \frac{\be_k}{B} \sum_{i \in \B_k} \left[\tilde{\nabla} f_i\!\left(w_k-\al \tilde{\nabla} f_i(w_k,\D_{in}^i),\D_{o}^i\right)  -\alpha  d_k^i \right] 
\end{align}
where $\alpha$ is the step size for each task, $\beta_k$ is the stepsize for the meta update, and the vectors $d_k^i$ are defined as 
{\small
\begin{align}\label{d_HF_MAML}
& d_k^i :=  \frac{1}{2\delta_k^i} \left ( \tilde{\nabla} f_i \!\left(w_k\!+\!\delta_k^i \tilde{\nabla} f_i(w_k\!-\!\al \tilde{\nabla} f_i(w_k,\D_{in}^i),\D_{o}^i), \D_{h}^i\right) \right. \nonumber \\
& \left. -\!\tilde{\nabla} f_i\!\left(w_k\!-\!\delta_k^i \tilde{\nabla} f_i(w_k\!-\!\al \tilde{\nabla} f_i(w_k,\D_{in}^i),\D_{o}^i),\D_{h}^i\right) \right ).
\end{align}}
Note that $d_k^i$ is an approximation for the term $ \tilde{\nabla}^2 f_i(w_{k},\D_{h}^i) \tilde{\nabla} f_i(w_k\!-\!\al \tilde{\nabla} f_i(w_k,\D_{in}^i),\D_{o}^i)$ which appears in $\nabla F$. In addition, $\delta_k^i>0$ indicates the accuracy of the Hessian-vector product approximation.  
As depicted in Algorithm~\ref{Algorithm3}, this update can be implemented efficiently in two stages similar to MAML. 
\begin{algorithm}[t!]
    \SetKwInOut{Input}{Input}
    \SetKwInOut{Output}{Output}
	\While{not done}{
    Choose a batch of \textit{i.i.d.} tasks $\B_k \subseteq \I$ with distribution $p$ and with size $B = |\B_k|$;\\
    \For{all $\T_i$ with $i \in \B_k$}{
    Compute  $\tilde{\nabla}f_i(w_k,\D_{in}^i)$ using dataset  $\D_{in}^i$;\\
    Set $w_{k+1}^i = w_k - \al \tilde{\nabla}f_i(w_k,\D_{in}^i)$;\\
    }
Compute $w_{k+1}$ according to the update \eqref{HF_MAML_main_update}; \\
    $k \gets k + 1$;
    }
    \caption{\label{Algorithm3} Hessian-free MAML (HF-MAML)}
\end{algorithm}
\section{Theoretical Results}\label{sec:Theory}

In this section, we characterize the overall complexity of MAML, FO-MAML, and HF-MAML for finding a first-order stationary point of $F$ when the loss functions $f_i$ are nonconvex but smooth. 

\begin{definition}
A random vector $w_\eps\in \R^d$ is called an $\epsilon$-approximate first order stationary point (FOSP) for problem \eqref{main_prob} if it satisfies $
\E[ \Vert \nabla F(w_\eps) \Vert] \leq \eps$.
\end{definition}

Our goal in this section is to answer two fundamental questions for each of the three considered methods. Can they find an $\epsilon$-FOSP for arbitrary $\eps>0$? If yes, how many iterations is needed for achieving such point? Before answering these questions, we first formally state our assumptions.

\begin{assumption}\label{ass:boundedness}
$F$ is bounded below, $\min F(w)> -\infty$ and $\Delta\!:=\! (F(w_0) \!-\! \min_{w \in \R^d} F(w) )$ is bounded.
\end{assumption}
\begin{assumption}\label{asm_smooth}
For every $i \in \I$, $f_i$ is twice continuously differentiable and $L_i$-smooth, i.e.,
\begin{equation*}
\Vert \nabla f_i(w) - \nabla f_i(u) \Vert \leq L_i \Vert w - u \Vert. 
\end{equation*}
\end{assumption}

For the simplicity of analysis, in the rest of the paper, we mostly work with $L:= \max_{i} L_i$ which can be considered as a parameter for the Lipschitz continuity of the gradients $\nabla f_i$ for all $i \in \I$.

\begin{assumption}\label{asm_Hesian_Lip}
For every $i \in \I$, the Hessian $\nabla^2 f_i$ is $\rho_i$-Lipschitz continuous, i.e., for every $w,u \in \R^d$, i.e., 
\begin{equation*}
\Vert \nabla^2 f_i(w) - \nabla^2 f_i(u) \Vert \leq \rho_i \Vert w - u \Vert.
\end{equation*}
\end{assumption}

To simplify our notation we use $\rho:=\max \rho_i$ as the Hessians Lipschitz continuity parameter for all $i \in \I$.
Note that we do not assume any smoothness conditions for the global loss $F$ and all the required conditions are for the individual loss functions $f_i$. In fact, later we show that under the conditions in Assumption~\ref{asm_smooth}, the global loss $F$ may not be gradient-Lipschitz in general.

The goal of Meta-learning is to train a model based on a set of given tasks so that this model can be used for learning a new unseen task. However, this is only possible if the training tasks are somehow related to unseen (test) tasks. In the following assumption, we formalize this condition by assuming that the gradient  $\nabla f_i$, which is an unbiased estimator of the gradient  $\nabla f= \E_{i \sim p}[\nabla f_i(w)]$, has a bounded variance. 

\begin{assumption}\label{asm_bounded_var}
The variance of gradient $\nabla f_i(w)$ is bounded, i.e., for some $\sigma>0$ we have
\begin{equation}\label{bounded_var}
 \E_{i \sim p}[\| \nabla f(w) - \nabla f_i(w)\|^2] \leq \sigma^2.
\end{equation}
\end{assumption}

{Note that this assumption is less strict comparing to the bounded gradient assumption in \citep{finn19a, rajeswaran2019metalearning}. In addition, for strongly convex functions, this assumption is closely related to the one in \citep{balcan_ICML} which states the optimal point of loss functions of all tasks are within a ball where its radius quantifies the similarity.}

In the following assumption we formally state the conditions required for the stochastic approximations of the gradients $\nabla f_i(w, \theta)$ and Hessians $\nabla^2 f_i(w, \theta)$. 

\begin{assumption}\label{asm_bounded_var_i}
For any $i$ and any $w\in \R^d$, the stochastic gradients $\nabla f_i(w, \theta)$ and Hessians $\nabla^2 f_i(w, \theta)$ have bounded variance, i.e., 
\begin{align}
&\E_\theta[\| \nabla f_i(w, \theta) - \nabla f_i(w)\|^2] \leq \tilde{\sigma}^2, \label{bounded_var_i} \\
&\E_\theta[\| \nabla^2 f_i(w, \theta) - \nabla^2 f_i(w)\|^2] \leq \sigma_H^2, \label{bounded_hess_i}  
\end{align}

where $\tsigma$ and $\sigma_H$ are non-negative constants.
\end{assumption}


Finally, to simplify the statement of our results for MAML, FO-MAML, and HF-MAML, we make the following assumption on the relation of parameters. {Later in the appendix, we drop this assumption and state the \textit{general version} of our results.}
\begin{assumption}\label{asm_simplify}
We assume $\rho \alpha/L = \bigO(1)$. Also, we assume $\sigma^2 + \tsigma^2 = \bigO(1)$, where $\sigma$ and $\tsigma$ are defined in Assumptions \ref{asm_bounded_var} and \ref{asm_bounded_var_i}, respectively.
\end{assumption}


\subsection{Challenges in analyzing MAML methods}\label{AnalysisChallenges}

Before stating our main results for MAML, FO-MAML, and HF-MAML, in this subsection we briefly highlight some of the challenges that emerge in analyzing these algorithms and prove some intermediate lemmas that we will use in the following subsections. 

\noindent{\textbf{(I) Unbounded smoothness parameter:}} 
The global loss function $F$ that we are minimizing in the MAML algorithm by following a stochastic gradient descent step is not necessarily smooth over $\R^d$, and its smoothness parameter could be unbounded. We formally characterize the parameter for the Lipschitz continuity of the gradients $\nabla F$ in the following lemma.

\begin{lemma}\label{lem_smooth_F}
Consider the objective function $F$ defined in \eqref{main_prob} for the case that $\al \in [0, \frac{1}{L}]$. Suppose that the conditions in Assumptions~\ref{asm_smooth}-\ref{asm_Hesian_Lip} are satisfied. 
Then, for any $w,u \in \R^d$ we have
\begin{align}\label{smooth_F}
& \Vert \nabla F(w) - \nabla F(u) \Vert \leq \min \{ L(w), L(u)\} \Vert w \!-\! u \Vert.
\end{align}
where $L(w) := 4L \!+\! 2 \rho \al \E_{i \sim p} \|\nabla f_i(w)\|$.
\end{lemma}

The result in Lemma~\ref{lem_smooth_F} shows that the objective function $F$ is smooth with a parameter that depends on the minimum of the expected norm of gradients. In other words, when we measure the smoothness of gradients between two points $w$ and $u$, the smoothness parameter depends on $\min\{\E_{i \sim p} \|\nabla f_i(w)\|,\E_{i \sim p} \|\nabla f_i(u)\|\}$. Indeed, this term could be unbounded or arbitrarily large as we have no assumption on the gradients norm. Moreover, computation of $\min\{\E_{i \sim p} \|\nabla f_i(w)\|,\E_{i \sim p} \|\nabla f_i(u)\|\}$ could be costly as it requires access to the gradients of all tasks.

%
%
%

\noindent{\textbf{(II) Stochastic stepsize:}} For most optimization methods, including SGD, the stepsize is selected proportional to the inverse of the smoothness parameter. However, in our setting, this parameter depends on the norm of gradient of all tasks which is not computationally tractable. To resolve this issue, we propose a method for choosing the stepsize $\beta_k$ by approximating $L(w)$ with an average over a batch of tasks. Specifically, we approximate $\E_{i \sim p} \|\nabla f_i(w)\|$ in the definition of $L(w)$ using the estimator $\sum_{j \in \B'} \|\tilde{\nabla} f_j(w, \D_\beta^j)\|$ where $\D_\beta^j$ is a dataset corresponding to task $j$ with size $D_\beta$. Hence, we estimate $L(w)$ by

\begin{equation}\label{el_tilde_def}
\tilde{L}(w):={4L + \frac{2\rho \al}{B'} \sum_{j \in \B'} \|\tilde{\nabla} f_j(w, \D_\beta^j)\|  }.
\end{equation} 
Using this estimate, our stepsize $\beta_k$ is tuned to be a constant times the inverse of $\tilde{L}(w)$ which we denote by  $\tilde{\be}(w)=1/\tilde{L}(w)$, i.e., $\beta_k=c\tilde{\be}(w)=c/\tilde{L}(w)$. This simple observation shows that the stepsize that we need to use for MAML algorithms is stochastic as $1/\tilde{L}(w)$ is a random parameter and depends on the choice of $\B'$. Therefore, we need to derive lower and upper bounds on the  expectations $E[\beta_k]$ and $\E[\beta_k^2]$, respectively, as they appear in the convergence analysis of gradient-based methods. Considering the defintion $\beta_k=c\tilde{\be}(w)$, we state these bounds for $\tilde{\be}(w)$ in the following lemma.

\begin{lemma}\label{tilde_beta}
Consider the function $F$ defined in \eqref{main_prob} for the case that $\al \in [0, \frac{1}{L}]$. Suppose Assumptions~\ref{asm_smooth}-\ref{asm_bounded_var_i} hold.
Further, consider the definition 

\begin{equation}\label{beta_update}
\tilde{\be}(w) \! :=\! \frac{1}{\tilde{L}(w)}\!:=\! \frac{1}{4L\! +\! 2 \rho \al \sum_{j \in \B'} \|\tilde{\nabla} f_j(w, \D_\beta^j)\| /{B'} },
\end{equation}
where $\B'$ is a batch of tasks with size $B'$ which are independently drawn with distribution $p$, and for every $j \in \B'$, $\D_\beta^j$ is a dataset corresponding to task $j$ with size $D_\beta$. If the conditions 

\begin{equation}\label{size_B'}
B' \geq \ceil[\big]{0.5 \left({\rho \al \sigma}/{L}\right)^2}, \quad D_\beta \geq \ceil[\big]{\left({2 \rho \al \tsigma}/{L}\right)^2}
\end{equation}
are satisfied, then we have

\begin{equation} \label{first_second_moment}
\E \left[\tilde{\be}(w)\right] \geq  \frac{0.8}{L(w)}, \quad \E\left[\tilde{\be}(w)^2\right] \leq \frac{3.125}{L(w)^2}
\end{equation}
where $L(w)= 4L + 2 \rho \al \E_{i \sim p} \|\nabla f_i(w)\|$.
\end{lemma}

Lemma \ref{tilde_beta} shows that if we set
 $\beta_k=c\tilde{\beta}(w_k)$, with $\tilde{\beta}(w_k)$ given
in \eqref{beta_update} and the batch-sizes $B'$ and $D_\beta$ satisfy the conditions \eqref{size_B'}, then the first moment of  $\beta_k$ is bounded below by a factor of $1/L(w_k)$ and its second moment is upper bounded by a factor of $1/L(w_k)^2$. 

Throughout the paper, we assume at each iteration $k$, the batches $\B'_k, \{\D_\beta^j\}_{j \in \B'_k}$ are independently drawn from $\B_k$ and $\{\D_{in}^i, \D_o^i, \D_h^i\}_{i \in \B_k}$ used in the updates of MAML methods.
Also, it is worth emphasizing that the batch size for the random sets $\B'_k$ and $ \{\D_\beta^j\}_{j \in \B'_k}$ are independent of the desired accuracy $\epsilon$ and the extra cost for the computation of these batches is of $\mathcal{O}(1)$.

\noindent{\textbf{(III) Biased estimator:}} The statement that MAML performs an update of stochastic gradient descent at each iteration on the objective function $F$ is not quite accurate. To better highlight this point, recall the update of MAML in \eqref{hassan}. According to this update, the descent direction $g_k$ for MAML at step $k$ is given by

\begin{equation*}
g_k:= \frac{1}{B} \sum_{i \in \B_k}  A_{i,k}\ \! \tilde{\nabla} f_i\!\left(w_k-\al \tilde{\nabla} f_i(w_k,\D_{in}^i),\D_{o}^i\right),
\end{equation*} 
with $A_{i,k}:=(I-\al \tilde{\nabla}^2 f_i(w_{k},D_{h}^i))$, while the exact gradient of $F$ at $w_k$ is given by 
\begin{equation*}
\nabla F(w_k)\!= \!\E_{i \sim p} \left [\! \left(I\!-\!\alpha \nabla^2 f_i(w_k)\right)  \!\nabla f_i(w_k\!-\!\al \nabla f_i(w_k))\! \right ]\!.
\end{equation*} 
Given $w_k$, $g_k$ is not an unbiased estimator of the gradient $\nabla F(w_k)$ as the stochastic gradient $\tilde{\nabla} f_i(w_k,\D_{in}^i)$ is within the stochastic gradient $\tilde{\nabla} f_i (w_k-\al \tilde{\nabla} f_i(w_k,\D_{in}^i),\D_{o}^i)$. Hence, the descent direction that we use in the update of MAML for updating models is a biased estimator of $\nabla F(w_k)$. This is another challenge that we face in analyzing MAML and its variants. To overcome this challenge, we need to characterize the first-order and second-order moments of the expression $\tilde{\nabla} f_i (w_k-\al \tilde{\nabla} f_i(w_k,\D_{in}^i),\D_{o}^i)$.

\begin{lemma}\label{lemma:moments}
Consider $F$ in \eqref{main_prob} for the case that $\al \in [0, \frac{1}{L}]$. Suppose Assumptions~\ref{asm_smooth}-\ref{asm_bounded_var_i} hold.
Then, 

\begin{align}\label{bias}
&\E_{\D_{in},\D_o}[\tilde{\nabla} f_i (w_k-\al \tilde{\nabla} f_i(w_k,\D_{in}^i),\D_{o}^i) \mid \F_k]= \nabla f_{i} \left  (w_k \!-\! \alpha \nabla f_{i}(w_k) \right ) +e_{i,k} ,\  \text{where} \   \|e_{i,k} \| \leq \frac{\al L\tsigma}{\sqrt{D_{in}}}.
\end{align}
Moreover, for arbitrary $\phi>0$ we have
\begin{align}\label{sec_moment}
 \E_{\D_{in},\D_o} \!\left [ \|\tilde{\nabla} f_i (w_k\!-\!\al \tilde{\nabla} f_i(w_k,\D_{in}^i),\D_{o}^i)\|^2  \mid \F_k \right]  \leq & \left(1+\frac{1}{\phi}\right)\|\nabla f_{i} (w_k - \alpha \nabla f_{i}(w_k))\|^2\nonumber\\
  & +  \frac{(1\!+\!\phi) \al^2 L^2\tsigma^2}{D_{in}}+ \frac{\tsigma^2}{D_o}.
\end{align}
\end{lemma}

The result in Lemma~\ref{lemma:moments} clarifies the reason that the descent direction of MAML denoted by $g_k$ is a biased estimator of $\nabla F(w_k)$. It shows that the bias is bounded above by a constant which depends on the variance of the stochastic gradients $\tilde{\nabla} f_i$ and the stepsize $\alpha$ for the inner steps. By setting $\alpha=0$, the vector $\tilde{\nabla} f_i (w_k-\al \tilde{\nabla} f_i(w_k,\D_{in}^i),\D_{o}^i)$ becomes an unbiased estimate of $\nabla f_{i} (w_k - \alpha \nabla f_{i}(w_k) )$ as our result in \eqref{bias} also suggests. Also, the result in \eqref{sec_moment} shows that the second moment of $\tilde{\nabla} f_i (w_k-\al \tilde{\nabla} f_i(w_k,\D_{in}^i),\D_{o}^i)$ is bounded above by the sum of a multiplicand of  $\|\nabla f_{i} (w_k - \alpha \nabla f_{i}(w_k) )\|^2$ and a multiplicand of $\tilde{\sigma}^2$. 



\subsection{On the Connection of $F$ and $\hat{F}$}

In this subsection, we investigate the connection between $F$ and $\hat{F}$ defined in \eqref{main_prob} and \eqref{main_prob_2}, respectively. In particular, in the following theorem, we characterize the difference between their gradients. Later, using this result, we show all the methods that we study achieve the same level of gradient norm with respect to both $F$ and $\hat{F}$, up to some constant.
\begin{theorem}\label{Thm_F_hatF}
Consider the functions $F$ and $\hat{F}$ defined in \eqref{main_prob} and \eqref{main_prob_2}, respectively, for the case that $\al \in (0,\frac{1}{L}]$. Suppose Assumptions~\ref{asm_smooth}-\ref{asm_bounded_var_i} hold.
Then, for any $w \in \R^d$, we have

\begin{equation}
\|\nabla \hat{F}(w) - \nabla F(w)\| \leq 2 \alpha L \frac{\tsigma}{\sqrt{D_{test}}} + \alpha^2 L \frac{\sigma_H \tilde{\sigma}}{D_{test}}.	
\end{equation}
\end{theorem}
Next, we mainly focus on characterizing the behavior of MAML, FO-MAML, and HF-MAML with respect to $F$, and by using the above theorem, we can immediately obtain bounds on the norm of $\nabla \hat{F}$ as well. In fact, the above theorem indicates the difference between $\nabla F$ and $\nabla \hat{F}$ is $\bigO (\max\{\frac{\tsigma}{\sqrt{D_{test}}}, \frac{\sigma_H \tilde{\sigma}}{D_{test}}\} )$.

\subsection{Convergence of MAML}

In this subsection, we study the overall complexity of MAML for finding an $\eps$-FOSP of the loss functions $F$ and $\hat{F}$ defined in \eqref{main_prob} and \eqref{main_prob_2}, respectively.  
\begin{theorem}\label{Thm_SGD_general_simple} 
Consider $F$ in \eqref{main_prob} for the case that $\al \in (0,\frac{1}{6L}]$. Suppose Assumptions~\ref{ass:boundedness}-\ref{asm_simplify} hold.
Consider running MAML with batch sizes satisfying the conditions $D_h \geq \ceil{2 \al^2 \sigma_H^2}$ and $B \geq 20$.
Let $\beta_k = \tilde{\beta}(w_k)/12$
where $\tilde{\beta}(w)$ is given in \eqref{beta_update}.
Then, for any $1>\eps >0$, MAML finds a solution $w_\eps$ such that

\begin{equation}
 \E[ \| \nabla F(w_\eps) \|] \leq \bigO \left ( \sqrt{\frac{\sigma^2}{B} + \frac{\tsigma^2}{B D_o} + \frac{\tsigma^2}{D_{in}}} \right ) + \eps	
\end{equation}
with a total number of iterations of

\begin{equation}
\bigO(1) \Delta \min \left \{ \frac{L}{\eps^2},
\frac{LB}{\sigma^2} + \frac{L(B D_o+D_{in})}{\tsigma^2} \right \} .
\end{equation}
\end{theorem}
The result in Theorem~\ref{Thm_SGD_general_simple}  shows that after running MAML for $\mathcal{O}(\frac{1}{\eps^2}+ \frac{B}{\sigma^2} + \frac{BD_o+D_{in}}{\tsigma^2})$ iterations, we can find a point $w^{\dag}$ that its expected gradient norm  $\E[ \| \nabla F(w^{\dag}) \|] $ is at most of $\eps+ \mathcal{O}( \sqrt{ \frac{\sigma^2}{B} +\frac{ \tsigma^2}{B D_o} + \frac{ \tsigma^2}{D_{in}}})$. This result implies that if we choose the batch sizes $B$, $D_o$, and $D_{in}$ properly (as a function of $\eps$), then for any $\eps>0$ it is possible to reach an $\eps$-FOSP of problem \eqref{main_prob} in a number of iterations which is polynomial in $1/\eps$. We formally state this result in the following corollary.

\begin{corollary}\label{cor_large_batch_MAML}
Suppose the condition in Theorem~\ref{Thm_SGD_general_simple} are satisfied. Then, if the batch sizes $B$, $D_o$, and $D_{in}$ satisfy the following conditions, 

\begin{equation}
B \geq  (C_1 {\sigma^2})/{\eps^2}, \quad D_{in}, BD_o \geq (C_2 {\tsigma^2})/{\eps^2},
\end{equation}
for some constants $C_1$ and $C_2$, then MAML finds an $\eps$-FOSP after $\Delta \bigO ({L}/{\eps^2})$ iterations.
\end{corollary}

The result shows that with sufficient samples for the batch of stochastic gradient evaluations, i.e.,  $D_{in}$ and $D_{o}$, and for the batch of tasks $B$, MAML finds an $\eps$-FOSP after at most $\mathcal{O}(1/\eps^2)$ iterations for any $\eps>0$. 
\begin{remark}
\sloppy Based on Theorem \ref{Thm_F_hatF}, the difference between $\nabla F$ and $\nabla \hat{F}$ is $\bigO \left (\max\{\frac{\tsigma}{\sqrt{D_{test}}}, \frac{\sigma_H \tilde{\sigma}}{D_{test}}\} \right )$. Given that, and since in practice, we usually choose $D_{test}$ at least as large as $D_{in}$, one can see that as long as $\sigma_H$ is not significantly larger than $\tilde{\sigma}$, the order of norm of gradient for both $F$ and $\hat{F}$ would be similar for all the results, up to some constant. This argument holds for FO-MAML and HF-MAML as well. 	
\end{remark}

\subsection{Convergence of FO-MAML}

Now we proceed to characterize the convergence of the first order approximation of MAML (FO-MAML). 
\begin{theorem}\label{First_order_MAML_det_simple}
Consider $F$ in \eqref{main_prob} for the case that $\al \in (0,\frac{1}{10L}]$. Suppose Assumptions~\ref{ass:boundedness}-\ref{asm_simplify} hold. 
Consider running FO-MAML with batch sizes satisfying the conditions $D_h \geq \ceil{2 \al^2 \sigma_H^2}$ and $B \geq 20$.
Let $\beta_k = \tilde{\beta}(w_k)/18$
where $\tilde{\beta}(w)$ is defined in \eqref{beta_update}.
Then, for any $1>\eps >0$, FO-MAML finds $w_\eps$ such that
\begin{equation}\label{result_FOMAML_general_simple}
\E  [ \| \nabla F(w_\eps) \|] \!\leq \!\bigO\! \left (\! \sqrt{\sigma^2\! \left( \!\alpha^2 L^2\!+\!\frac{1}{B}\!\right)\! +\! \frac{\tsigma^2}{B D_o} \!+\! \frac{\tsigma^2}{D_{in}}} \right )\!	+\! \eps
\end{equation}
with a total number of iterations of

\begin{equation*}
\bigO(1) \Delta \min \left \{ \frac{L}{\eps^2}, \frac{L}{\sigma^2 (\alpha^2 L^2\!+\!{B^{-1}})}+\frac{L(BD_o\!+\!D_{in})}{\tsigma^2} \right \}\!.
\end{equation*}
\end{theorem}
Comparing Theorem~\ref{First_order_MAML_det_simple} with Theorem \ref{Thm_SGD_general_simple} implies that FO-MAML, in contrast to MAML, may not converge to an exact first-order stationary point even when we use large batch sizes (see the subsection below). Specifically, even if we choose large batch sizes $B$, $D_{in}$, and $D_o$ for FO-MAML, the  gradient norm cannot become smaller than $\mathcal{O}(\alpha \sigma)$. This is because of the $\alpha^2 L^2 \sigma^2$ term in \eqref{result_FOMAML_general_simple} which does not decrease by increasing the batch sizes for the tasks and stochastic gradient evaluations.
Now we state the results for FO-MAML when, as in corollary \ref{cor_large_batch_MAML}, we use batch sizes of $\mathcal{O}(1/\eps^2)$. 

\begin{corollary}
Suppose the condition in Theorem~\ref{First_order_MAML_det_simple} are satisfied. Then, if the batch sizes $B$, $D_o$, and $D_{in}$ satisfy the following conditions, 

\begin{equation}
B \geq  C_1 \frac{1}{\alpha^2 L^2}, \quad D_{in}, BD_o \geq C_2 \frac{\tsigma^2}{\alpha^2 \sigma^2 L^2},
\end{equation}
\sloppy for some constants $C_1$ and $C_2$,
then FO-MAML finds a point $w^{\dag}$ satisfying the condition $\E[\|\nabla F(w^{\dag})\|]\leq \mathcal{O}(\alpha \sigma L)$,  after at most 
$\Delta \bigO({1}/{(\alpha^2 \sigma^2 L)})$ iterations.
\end{corollary}

\subsubsection{Convex Quadratic Case}\label{Qudratic_MAML}
Theorem \ref{First_order_MAML_det_simple} suggests that FO-MAML might not be able to achieve any arbitrary level of accuracy, even when exact gradients and Hessians are available. In this subsection, we provide an example to show that this is indeed the case. In particular, we consider the case that we have $n$ equally likely tasks, where the loss function corresponding to task $i$ is given by 
\begin{equation}
f_i(w) = \frac{1}{2} w^\top A_i w + b_i^\top w + c_i	
\end{equation} 
where $c_i \in \R$, $b_i \in \R^d$, , and $A_i \in \R^{d \times d}$ is a symmetric and positive definite matrix with $\|A_i\| \leq L$. Note that, in this case, we have
\begin{equation}
\nabla f_i(w) = A_i w + b_i, \quad \nabla^2 f_i(w) = A_i. 	
\end{equation}
Also, throughout this section, we assume we have access to the exact value of Hessians and gradients, and also we go over all tasks at each iteration, i.e., $\B_k$ is equal to the set of all tasks. We further assume $\alpha < 1/L$.

\vspace{2mm}

\noindent $\bullet$ \textbf{Solution of MAML Problem}

Let us first derive the solution to problem \eqref{main_prob}. Setting $\nabla F(w^*)$ equal to $0$ implies that
\begin{equation}
\frac{1}{n} \sum_{i=1}^n (I - \alpha A_i) \left ( A_i \left (w^* - \alpha (A_i w^* + b_i) \right ) + b_i\right ) = 0.
\end{equation}
Simplifying this equation yields
\begin{equation}\label{w_MAML_quad}
\left ( \frac{1}{n} \sum_{i=1}^n (I - \alpha A_i)^2 A_i \right ) w^* = - \left ( \frac{1}{n} \sum_{i=1}^n (I - \alpha A_i)^2 b_i \right ).	
\end{equation}
Note that, for any $i$, $(I - \alpha A_i)^2 A_i$ is positive definite. This is due to the fact that, for any $u \in \R^d$, we have
\begin{align*}
u^\top 	(I - \alpha A_i)^2 A_i u = \left ( (I - \alpha A_i) u \right )^\top A_i \left ( (I - \alpha A_i) u \right ) > 0
\end{align*}
where the first equality follows from the fact that $A_i$ is symmetric and the last inequality is obtained using positive definiteness of $A_i$. As a result, and since sum of positive definite matrices is also positive definite, $\left ( \frac{1}{n} \sum_{i=1}^n (I - \alpha A_i)^2 A_i \right )$ is positive definite, and hence invertible. Thus, from \eqref{w_MAML_quad}, we have
\begin{equation}\label{w_MAML_quad_2}
w^* = 	- \left ( \frac{1}{n} \sum_{i=1}^n (I - \alpha A_i)^2 A_i \right )^{-1} \left ( \frac{1}{n} \sum_{i=1}^n (I - \alpha A_i)^2 b_i \right ).
\end{equation}

\vspace{2mm}

\noindent $\bullet$ \textbf{What FO-MAML Converges to?}

FO-MAML update is given by
\begin{align}\label{w_FOMAML_quad}
w_{k+1} &= w_k - \alpha 	\frac{1}{n} \sum_{i=1}^n \left ( A_i \left (w_k - \alpha (A_i w_k + b_i) \right ) + b_i\right ) \nonumber\\
& = \left ( I - \alpha \frac{1}{n} \sum_{i=1}^n (I - \alpha A_i) A_i \right ) w_k - \alpha \frac{1}{n} \sum_{i=1}^n (I - \alpha A_i) b_i.
\end{align}
Note that, since $\alpha < 1/L$, for any $i$, $I-\alpha A_i$ is symmetric and positive definite. Note that since both $I-\alpha A_i$ and $A_i$ are positive definite and $(I-\alpha A_i)A_i$ can be written as $A_i^{1/2} (I-\alpha A_i) A_i^{1/2}$, then we can show that $(I-\alpha A_i)A_i$ is also positive definite. In addition, since $\|A_i\| \leq L$ and $\|I - \alpha A_i\| \leq 1$, we have $\| (I-\alpha A_i)A_i\| \leq L$. Therefore, the term $\frac{1}{n} \sum_{i=1}^n (I - \alpha A_i) A_i$ is  positive definite and its norm is upper bounded by $L$. Since $\alpha < 1/L$, this implies that $I - \alpha \frac{1}{n} \sum_{i=1}^n (I - \alpha A_i) A_i$ is positive definite, and thus invertible.

Using this result, it is immediate to see $\{w_k\}_k$ \eqref{w_FOMAML_quad} converges to
\begin{equation}\label{w_FOMAML_quad_2}
w_{FO} = - \left ( \frac{1}{n} \sum_{i=1}^n (I - \alpha A_i) A_i \right )^{-1} \left ( \frac{1}{n} \sum_{i=1}^n (I - \alpha A_i) b_i \right ).	
\end{equation}
Comparing \eqref{w_MAML_quad_2} and \eqref{w_FOMAML_quad_2}, it is clear that FO-MAML does not converge to the solution of MAML problem denoted by $w^*$. However, if either $\alpha$ is very small or if $A_i$ are all close to some $A$ and $b_i$ are all close to some $b$ which leads to $\sigma$ being very small, then $w_{FO}$ would be very close to $w^*$, which is in line with what we observed in Theorem \ref{First_order_MAML_det_simple}. 

\subsection{Convergence of HF-MAML}

Now we proceed to analyze the overall complexity of our proposed HF-MAML method.
\begin{theorem}\label{Thm_HF_MAML_simple}
Consider the function $F$ defined in \eqref{main_prob} for the case that $\al \in (0,\frac{1}{6L}]$. Suppose Assumptions~\ref{ass:boundedness}-\ref{asm_simplify} hold. Consider running HF-MAML with batch sizes satisfying the conditions $D_h \geq \ceil{36 (\al \rho \tsigma)^2}$ and $B \geq 20$.
Let $\beta_k = \tilde{\beta}(w_k)/25$
where $\tilde{\beta}(w)$ is defined in \eqref{beta_update}. Also, we choose the approximation parameter $\delta_k^i$ in HF-MAML as

\begin{equation*}
\delta_k^i = \frac{1}{6 \rho \alpha \| \tilde{\nabla} f_i(w_k\!-\!\al \tilde{\nabla} f_i(w_k,\D_{in}^i),\D_{o}^i) \|}.
\end{equation*}
Then, HF-MAML finds a solution $w_\eps$ such that

\begin{equation}
 \E[ \| \nabla F(w_\eps) \|] \leq \bigO \bigg( \sqrt{\frac{\sigma^2}{B} + \frac{\tsigma^2}{B D_o} + \frac{\tsigma^2}{D_{in}}} \bigg) + \eps	
\end{equation}
with a total number of iterations of

\begin{equation}
\bigO(1) \Delta \min \left \{ \frac{L}{\eps^2},
\frac{LB}{\sigma^2} + \frac{L(B D_o+D_{in})}{\tsigma^2} \right \} .
\end{equation}
\end{theorem}


Comparing the results in Theorem~\ref{Thm_HF_MAML_simple} for HF-MAML with the result in Theorem~\ref{Thm_SGD_general_simple} for MAML shows that the complexity of these methods and the resulted accuracy are the same, up to a constant factor. Hence, HF-MAML recovers the complexity of MAML \textit{without} computing second-order information or performing any update that has a complexity of $\mathcal{O}(d^2)$. 
Also, as stated after Theorem \ref{Thm_HF_MAML_simple}, the following result holds as an immediate consequence:
\begin{corollary}\label{cor_large_batch_HF_MAML}
Suppose the condition in Theorem~\ref{Thm_HF_MAML_simple} are satisfied. Then, if the batch sizes $B$, $D_o$, and $D_{in}$ satisfy the following conditions, 
\begin{equation*}
B \geq  C_1 \frac{\sigma^2}{\eps^2}, \quad D_{in}, BD_o \geq C_2 \frac{\tsigma^2}{\eps^2},
\end{equation*}
for some constants $C_1$ and $C_2$, then the iterates generated by HF-MAML finds an $\eps$-FOSP, i.e., $\E[\|\nabla F(w)\|]\leq \eps$,  after $\Delta \bigO ({L}/{\eps^2})$ iterations.
\end{corollary}


\section{A Numerical Example}\label{Numerical}

 \begin{figure}
   \centering
     \begin{subfigure}[H]{0.45\textwidth}
       \includegraphics[width=1\columnwidth]{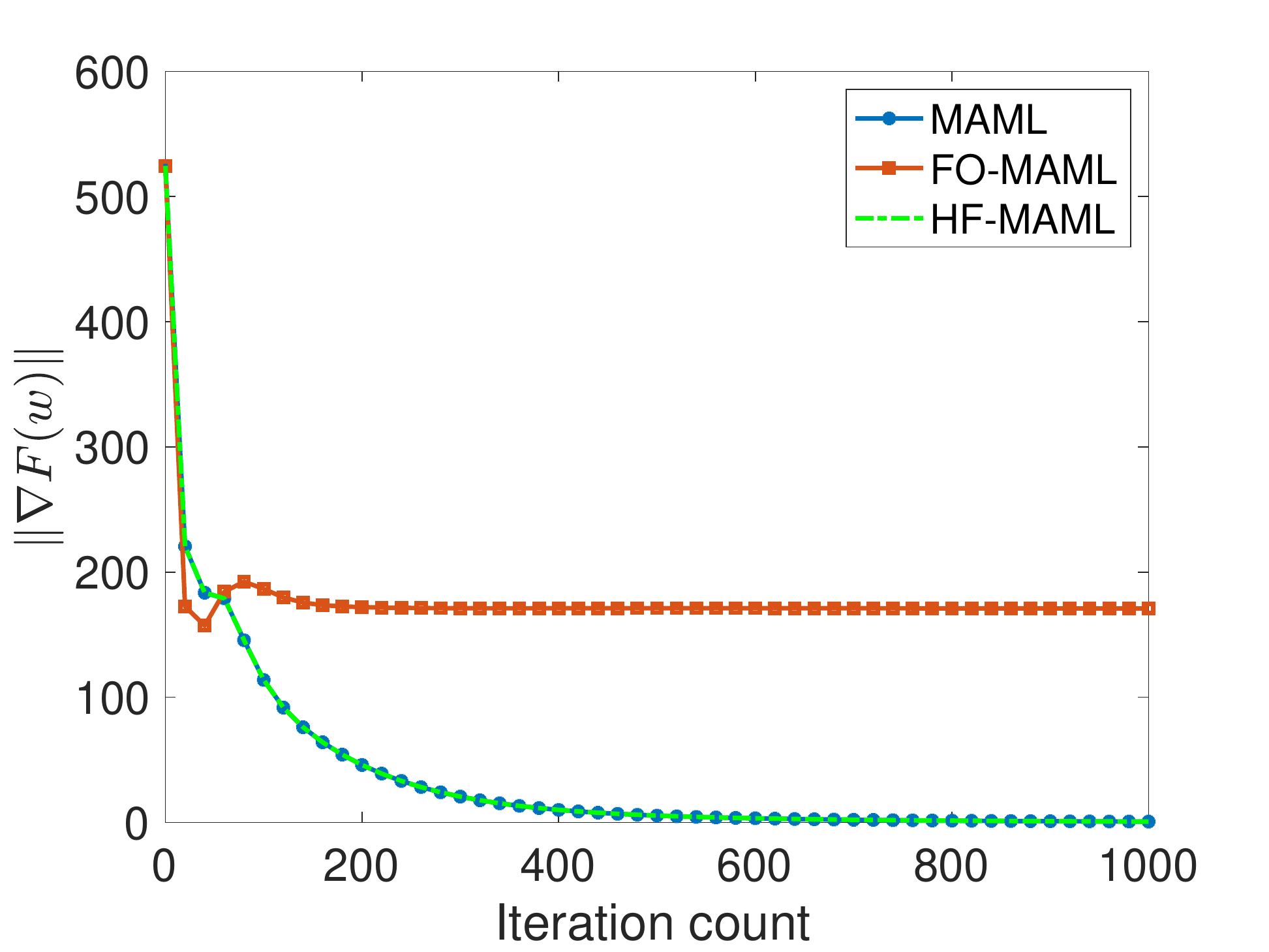}
             \subcaption{Having exact gradients \& Hessians available}
       \label{Fig_1}
       \end{subfigure}
       \qquad 
     \begin{subfigure}[H]{0.45\textwidth}
      \includegraphics[width=1\columnwidth]{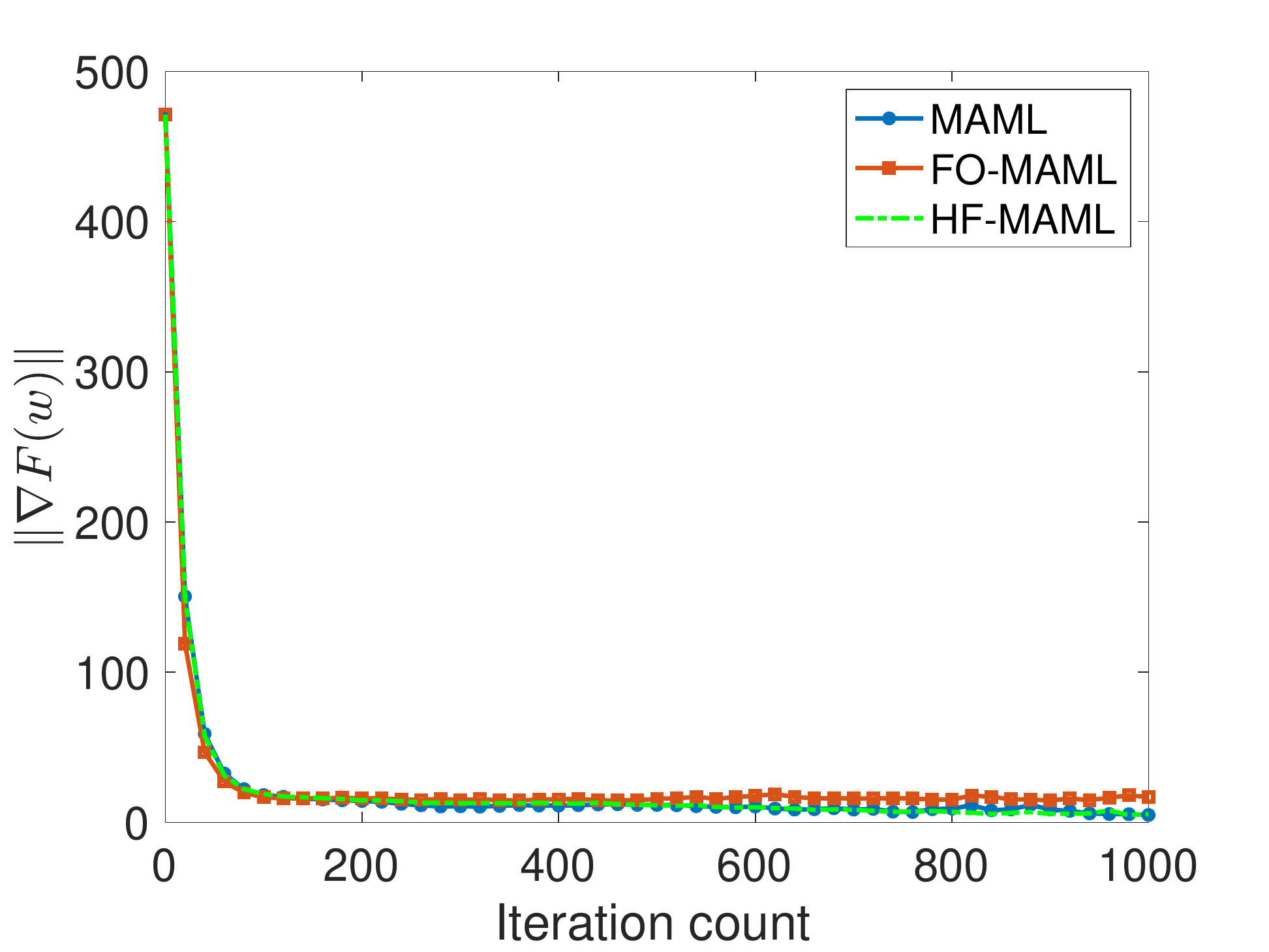}
       \subcaption{Stochastic gradients \& Hessians with similar tasks}
       \label{Fig_2}
   \end{subfigure}\\
     \begin{subfigure}[H]{0.45\textwidth}
       \includegraphics[width=1\columnwidth]{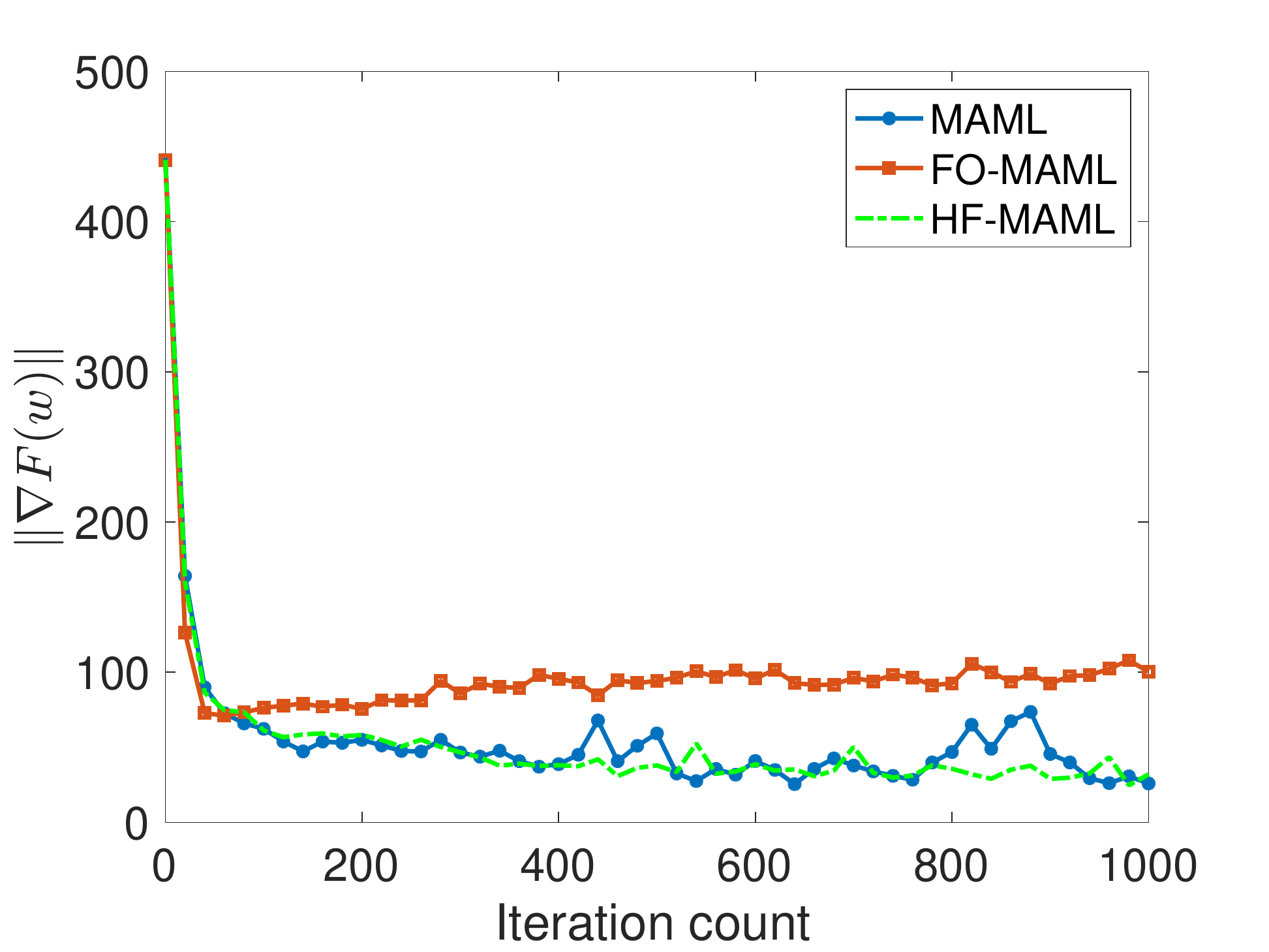}
             \subcaption{Stochastic gradients \& Hessians with less similar tasks}
       \label{Fig_3}
       \end{subfigure}
     \caption{Comparison of MAML, FO-MAML, and HF-MAML for a 1-rank matrix factorization problem}\label{Fig_Main} 
  \end{figure}
We consider a 1-rank matrix factorization problem where the loss for task $i$ is 
\begin{equation*}
f_i(x) = \frac{1}{4} \| x x^\top - M_i \|_F^2	
\end{equation*}
where $\|.\|_F$ is the Frobenius norm, and, $M_i$ is a rank one matrix, generated as $M_i=g_i g_i^\top$ where $g_i$ is a random zero-mean Gaussian vector. The variance of $g_i$ controls task similarity ($\sigma^2$). To capture noisy gradients and Hessians we add random Gaussian noise with variance $\tsigma^2$ to the tasks gradients and Hessians.

Recall from our results that the best accuracy for MAML and HF-MAML is $\bigO(\tilde{\sigma}/\sqrt{D})$ while the one for FO-MAML is $\bigO(\alpha \sigma + \tilde{\sigma}/\sqrt{D})$, which has the additional term $\bigO(\alpha \sigma)$, and hence, it does not go to zero even with exact gradients and Hessians.
\begin{itemize}
\item In Figure \ref{Fig_1}, we assume gradients and Hessians are exact ($\tilde{\sigma}=0$) and focus on task variation. We assume that the batch-size for tasks $B$ equal to the number of tasks which is 20. In this case, even though the gradients are exact, FO-MAML converges to an error level with a gap compared to two others. This is consistent with our results.
\item We next consider noisy gradients and Hessians. We also set $B$=10 and number of tasks equal to 50. In Figure \ref{Fig_2}, we choose the variance of $g_i$ small to ensure the tasks are relatively similar. Here the additional persistent error of FO-MAML is negligible compared to the other term and all three methods behave similarly.
\item In Figure \ref{Fig_3}, we increase the variance of $g_i$, i.e., tasks are less similar. In this case, the $\bigO(\alpha \sigma)$ term dominates the other term in the error of FO-MAML, and it has worse performance compared to two others.
\end{itemize}


\section{Conclusion}
In this work, we studied the convergence properties of MAML, its first-order approximation (FO-MAML), and our proposed Hessian-free MAML (HF-MAML) for non-convex functions. In particular, we characterized their best achievable accuracy in terms of gradient norm when we have access to enough samples and further showed their best possible accuracy when the number of available samples is limited. Our results indicate that MAML can find an $\eps$-first-order stationary point, for any positive $\eps$ 
{at the cost of using the second-order information of loss functions.}
On the other hand, we illustrated that although the iteration cost of FO-MAML is $ \bigO(d)$, it cannot reach any desired level of accuracy. That said, we next showed that  HF-MAML has the best of both worlds, i.e., it {does not require access to the second-order derivative and} has a cost of $\bigO(d)$ at each iteration, while it can find an $\eps$-first-order stationary point, for any positive $\eps$.


\section{Acknowledgment}
Research was sponsored by the United States Air Force Research Laboratory and was accomplished under Cooperative Agreement Number FA8750-19-2-1000. The views and conclusions contained in this document are those of the authors and should not be interpreted as representing the official policies, either expressed or implied, of the United States Air Force or the U.S. Government. The U.S. Government is authorized to reproduce and distribute reprints for Government purposes notwithstanding any copyright notation herein. Alireza Fallah acknowledges support from MathWorks Engineering Fellowship. The authors would like to thank Chelsea Finn and Zhanyu Wang for their comments on the first draft of this paper.

\vspace{6mm}


\appendix

\begin{center}
\textbf{\huge{Appendix}}
\end{center}


\section{Intermediate results}
In this subsection, we prove some results that will be used in the rest of our proofs.

First, note that since we also assume that the functions $f_i$ are twice differentiable, the $L_i$-smoothness assumption also implies that for every $w,u \in \R^d$ we have
\begin{subequations} \label{smooth_2}
\begin{align}
&-L_i I_d \preceq \nabla^2 f_i(w) \preceq L_i I_d \quad \forall w \in \R^d,  \label{smooth_2:a}\\
& f_i(w) - f_i(u) \!- \!\nabla f_i(u)^\top (w-u) \!\leq\! \frac{L_i}{2} \|w-u\|^2. \label{smooth_2:b}
\end{align}
\end{subequations}
We use these relations in the subsequent analysis. 
Next, we use Lemma~\ref{lem_smooth_F} to show the following result which is analogous to \eqref{smooth_2:b} for $F$. We skip the proof as it is very similar to the proof of Lemma 1.2.3 in \citep{nesterov_convex}.
\begin{corollary}\label{bergman_bound}
Let $\al \in [0, \frac{1}{L}]$. Then, for  $w,u \in \R^d$,
\begin{equation}\label{descent_ineq}
F(u) - F(w) - \nabla{F}(w)^\top (u-w) \leq \frac{L(w)}{2} \|u-w \|^2,
\end{equation}
where $L(w)  = 4L + 2 \rho \al \E_{i \sim p} \|\nabla f_i(w)\|$. 
\end{corollary}
Now, we state the following theorem from \citep{wooff1985bounds}.
\begin{theorem}\label{inter_thm}
Let $X$ be random variable with left extremity zero, and let $c$ be a positive constant. Suppose that $\mu_{\scaleto{X}{4pt}} = \E[X]$ and $\sigma_{\scaleto{X}{4pt}}^2 = \text{Var }(X)$ are finite. Then, for every positive integer $k$, 
\begin{equation}	
\frac{1}{(\mu_{\scaleto{X}{4pt}} + c)^k} \leq \E\left[\frac{1}{(X+c)^k}\right] \leq \frac{\sigma_{\scaleto{X}{4pt}}^2/c^k + \mu_{\scaleto{X}{4pt}}^2 \gamma^k}{\sigma_{\scaleto{X}{4pt}}^2 + \mu_{\scaleto{X}{4pt}}^2}	
\end{equation}
where $\gamma = \mu_{\scaleto{X}{4pt}}/(\sigma_{\scaleto{X}{4pt}}^2 + \mu_{\scaleto{X}{4pt}}(\mu_{\scaleto{X}{4pt}}+c))$.
\end{theorem}
\begin{proof}
See Theorem 1 in \citep{wooff1985bounds}.	
\end{proof}
\begin{lemma}\label{lemma_bound_by_F}
Consider the definitions of $f$ in \eqref{f_def} an $F$ in \eqref{main_prob} for the case that $\al \in [0,\frac{\sqrt{2}-1}{L})$. Suppose that the conditions in Assumptions~\ref{asm_smooth}-\ref{asm_bounded_var} are satisfied. Further, recall the definitions $L:=\max L_i$ and $\rho:=\max \rho_i$. Then, for any $w \in \R^d$ we have
\begin{align}
& \| \nabla{f}(w) \| \leq C_1 \| \nabla{F}(w) \| + C_2 \sigma, \label{bound_f_by_F} \\  
&\E_{i \sim p}[\|\nabla F_i(w)\|^2] \leq 2(1+\al L)^2 C_1^2 \|\nabla F(w)\|^2 + (1+\al L)^2 (2 C_2^2+1) \sigma^2, \label{bound_F_i_by_F}     
\end{align}
where 
\begin{equation*}
C_1 = \frac{1}{1- 2 \al L - \al^2 L^2}, \quad C_2 = \frac{2 \al L + \al^2 L^2}{1-2 \al L - \al^2 L^2}.    
\end{equation*}
\end{lemma}
\begin{proof}
The gradient of the function $F(w)$ is given by
\begin{subequations}\label{gradient_F}
\begin{align}
\nabla F(w) &= \E_{i \sim p}[ \nabla F_i(w)], \label{gradient_F:a}\\
\nabla F_i(w) &= A_i(w) \nabla f_i(w-\al \nabla f_i(w)) \label{gradient_F:b}
\end{align}
\end{subequations}
with $A_i(w) := (I-\al \nabla^2 f_i(w))$. Note that using the mean value theorem we can write the gradient $\nabla f_i(w-\al \nabla f_i(w)) $ as
\begin{align}\label{mean_value}
\nabla f_i(w-\al \nabla f_i(w))  &= \nabla f_i(w) -\al \nabla^2 f_i (\tilde{w}_i)\nabla f_i(w)\nonumber\\
&= (I -\al \nabla^2 f_i (\tilde{w}_i))\nabla f_i(w)
\end{align}
for some $\tilde{w}_i$ which can be written as a convex combination of $w$ and $w-\al \nabla f_i(w)$.
Using \eqref{gradient_F:b} and the result in \eqref{mean_value} we can write
\begin{align}\label{eq_lem2_1}
\nabla F_i(w) = A_i(w) \nabla f_i(w-\al \nabla f_i(w)) = A_i(w) A_i(\tilde{w}_i) \nabla{f_i}(w)   ,
\end{align}
where $A_i(\tilde{w}_i) := (I-\al \nabla^2 f_i(\tilde{w}_i))$. Now, we have
\begin{align}
\| \nabla{f}(w) \| = \| \E_{i \sim p} \nabla{f_i}(w) \| &= \| \E_{i \sim p} \left [ \nabla{F_i}(w) + \left(\nabla{f_i}(w) - \nabla{F_i(w)} \right) \right ] \| \nonumber \\
& \leq \| \E_{i \sim p} \nabla{F_i}(w) \| + \|\E_{i \sim p} \left [ \left (I - A_i(w) A_i(\tilde{w}_i) \right) \nabla{f_i}(w) \right ] \| \label{lem2_ineq1}\\
& \leq \|\nabla{F}(w) \| + \E_{i \sim p} \left [ \| I - A_i(w) A_i(\tilde{w}_i) \| \| \nabla{f_i}(w) \| \right ] ,\label{lem2_ineq2}
\end{align}
where \eqref{lem2_ineq1} is obtained by substituting $\nabla F_i(w)$ from \eqref{eq_lem2_1}. Next, note that
\begin{equation*}
\| I - A_i(w) A_i(\tilde{w}_i) \| = \| \al \nabla^2 f_i(w) + \al \nabla^2 f_i(\tilde{w}_i) + \al^2 \nabla^2 f_i(w) \nabla^2 f_i(\tilde{w}_i) \| \leq 2 \al L + \al^2 L^2     ,
\end{equation*}
where the last inequality can be shown by using \eqref{smooth_2:a} and triangle inequality. Using this bound in \eqref{lem2_ineq2} yields 
\begin{align}
\| \nabla{f}(w) \| & \leq \|\nabla{F}(w) \| + (2 \al L + \al^2 L^2) \E_{i \sim p} \| \nabla{f_i}(w) \| \nonumber \\
& \leq \|\nabla{F}(w) \| + (2 \al L + \al^2 L^2) \left ( \| \E_{i \sim p} \nabla{f_i}(w) \| + \E_{i \sim p} \left [ \| \nabla{f_i}(w) - \E_{i \sim p} \nabla{f_i}(w) \| \right ] \right )  \nonumber \\
& \leq \|\nabla{F}(w) \| + (2 \al L + \al^2 L^2) \left (\| \nabla{f}(w) \| + \sigma \right), \label{lem2_ineq3}
\end{align}
where \eqref{lem2_ineq3} holds since $\E_{i \sim p} \nabla{f_i}(w) = \nabla{f}(w)$, and also, by Assumption \ref{asm_bounded_var},
\begin{equation}\label{lem2_ineq4}
\E_{i \sim p} \left [ \| \nabla{f_i}(w) - \E_{i \sim p} \nabla{f_i}(w) \| \right ] \leq \sqrt{\E_{i \sim p} \left [ \| \nabla{f_i}(w) - \nabla{f}(w) \|^2 \right ]} \leq \sigma.    
\end{equation}
Finally, moving the  term  $\| \nabla{f}(w) \| $ from the right hand side of \eqref{lem2_ineq3} to the left hand side and dividing both sides by $1/(1- 2 \al L - \al^2 L^2)$ completes the proof of \eqref{bound_f_by_F}. To show \eqref{bound_F_i_by_F}, note that, using \eqref{eq_lem2_1}, and the fact that $\|A_i(w)\|\leq (1+\al L)$ and $\|A_i(\tilde{w})\|\leq (1+\al L)$ we can write
\begin{align*}
\E_{i \sim p}[\|\nabla F_i(w)\|^2] & \leq \E_{i \sim p}[\|A_i(w)\|^2 \|A_i(\tilde{w}_i)\|^2 \|\nabla f_i(w)\|^2]\\
&\leq (1+\al L)^2 \E_{i \sim p}[\|\nabla f_i(w)\|^2] \\
& \leq (1+\al L)^2 \left ( \|\nabla f(w)\|^2 + \sigma^2 \right ) \\
& \leq (1+\al L)^2 \left ( 2 C_1^2 \|\nabla F(w)\|^2 + 2C_2^2 \sigma^2 + \sigma^2 \right )
\end{align*}
where the last inequality follows from \eqref{bound_f_by_F} along with the fact that $(a+b)^2 \leq 2a^2 + 2b^2$.
\end{proof}
\section{Proof of Lemma \ref{lem_smooth_F}}\label{lem_smooth_F_proof}

By considering the definition $\nabla F(w) = \E_{i \sim p}[ \nabla F_i(w)]$ where $\nabla F_i(w) = (I-\al \nabla^2 f_i({w})) \nabla f_i(w-\al \nabla f_i(w))$ we can show that 
\begin{align}
\Vert \nabla F(w) - \nabla F(u) \Vert &\leq \sum_{i \in \I} p_i \Vert \nabla F_i(w) - \nabla F_i(u) \Vert  \nonumber \\
&\leq \sum_{i \in \I} p_i (\Vert \nabla f_i(w-\al \nabla f_i(w)) - \nabla f_i(u-\al \nabla f_i(u)) \Vert \label{term_1} \\ 
& \quad + \al \Vert \nabla^2 f_i(w) \nabla f_i(w-\al \nabla f_i(w)) - \nabla^2 f_i(u) \nabla f_i(u-\al \nabla f_i(u)) \Vert \label{term_2}
\end{align}
To show the desired result, it suffices to bound both terms in \eqref{term_1} and \eqref{term_2}. For \eqref{term_1}, we have
\begin{align}
\Vert \nabla f_i(w-\al \nabla f_i(w)) - \nabla f_i(u-\al \nabla f_i(u)) \Vert  & \leq L \| w- u + \al (\nabla f_i(w) - \nabla f_i(u)) \|  \nonumber \\
& \leq L (1+ \al L) \| w- u \|, \label{term_1_bound}
\end{align}
where we used the smoothness assumption in Assumption~\ref{asm_smooth} for both inequalities. 
To bound \eqref{term_2}, note that
\begin{align}
\Vert & \nabla^2 f_i(w) \nabla f_i(w-\al \nabla f_i(w)) - \nabla^2 f_i(u) \nabla f_i(u-\al \nabla f_i(u)) \Vert  \nonumber \\
& =  \Vert \nabla^2 f_i(w) \nabla f_i(w-\al \nabla f_i(w)) - \nabla^2 f_i(w) \nabla f_i(u-\al \nabla f_i(u))  \nonumber \\
& \quad + \nabla^2 f_i(w) \nabla f_i(u-\al \nabla f_i(u)) - \nabla^2 f_i(u) \nabla f_i(u-\al \nabla f_i(u)) \Vert  \nonumber \\
& \leq \| \nabla^2 f_i(w) \| \|  \nabla f_i(w-\al \nabla f_i(w)) -  \nabla f_i(u-\al \nabla f_i(u))\|  \nonumber \\
& \quad + \| \nabla^2 f_i(w) - \nabla^2 f_i(u) \| \| \nabla f_i(u-\al \nabla f_i(u)) \|  \nonumber \\
&  \leq \left ( L^2 (1+\al L)  + \rho \| \nabla f_i(u-\al \nabla f_i(u)) \| \right ) \| w-u \|, \label{ineq1_term_2}
\end{align}
where \eqref{ineq1_term_2} follows from \eqref{term_1_bound}, \eqref{smooth_2:a}, and Assumption \ref{asm_Hesian_Lip}. To bound the gradient term in \eqref{ineq1_term_2}, we use the mean value theorem which implies that
\begin{align*}
\nabla f_i(u-\al \nabla f_i(u)) = \left (I - \al \nabla^2 f_i(\tilde{u}_i)\right) \nabla f_i(u)   
\end{align*}
holds for some $\tilde{u}_i$ which can be written as a convex combination of $u$ and $u-\al \nabla f_i(u)$. As a result, and by using \eqref{smooth_2:a}, we obtain 
\begin{equation}\label{ineq2_term_2}
\| \nabla f_i(u-\al \nabla f_i(u)) \| \leq (1+ \al L) \| \nabla f_i(u) \|.    
\end{equation}
Next, plugging \eqref{ineq2_term_2} in \eqref{ineq1_term_2} leads to
\begin{equation}\label{term_2_bound}
\Vert \nabla^2 f_i(w) \nabla f_i(w-\al \nabla f_i(w)) - \nabla^2 f_i(u) \nabla f_i(u-\al \nabla f_i(u)) \Vert \leq \left ( L^2  + \rho \| \nabla f_i(u) \| \right ) (1+\al L) \| w-u \|.   
\end{equation}
Using bounds \eqref{term_1_bound} and \eqref{term_2_bound} in \eqref{term_1} and \eqref{term_2}, respectively, along with the fact that $\al L \leq 1$, yields
\begin{equation*}
\Vert \nabla F(w) - \nabla F(u) \Vert \leq (4L + 2 \rho \al \E_{i \sim p} \|\nabla f_i(u)\|) \Vert w - u \Vert.   
\end{equation*}
We can show a similar bound with $\nabla f_i(u)$ replaced by $\nabla f_i(w)$ in the right hand side, and these two together complete the proof.
\section{Proof of Lemma \ref{tilde_beta}}\label{tilde_beta_proof}
First, note that as $\tilde{\nabla} f_j(w, \D_\beta^j)=\frac{1}{D_\beta}\sum_{\theta\in \D_\beta^j}\tilde{\nabla} f_j(w, \theta)$ and each $\tilde{\nabla} f_j(w, \theta)$ is an unbiased estimator of $\nabla f_j(w)$ with a bounded variance of $\tilde{\sigma}^2$, then for each task $\T_j$ we have
\begin{align}
\E_{\D_\beta^j}[\|\tilde{\nabla} f_j(w, \D_\beta^j)- \nabla f_j(w)\|^2] \leq  \frac{\tsigma^2}{{D_\beta}},
\end{align}
 and, therefore, $\E_{\D_\beta^j}[\|\tilde{\nabla} f_j(w, \D_\beta^j)- \nabla f_j(w)\|] \leq  \frac{\tsigma}{\sqrt{D_\beta}}$ we can write
\begin{align} \label{D_j_beta_bound}
\|\nabla f_j(w)\| - \frac{\tsigma}{\sqrt{D_\beta}} \leq  \E_{\D_\beta^j} [\|\tilde{\nabla}f_j(w,\D_\beta^j)\|] \leq \|\nabla f_j(w)\| +	 \frac{\tsigma}{\sqrt{D_\beta}}. 
\end{align}
To derive a bound on the second moment of $\tilde{\be}(w)$, we use the result of Theorem~\ref{inter_thm} for $X={ 2 \rho \al \sum_{j \in \B'} \|\tilde{\nabla} f_j(w, \D_\beta^j)\| /{B'} }$, $c=4L$, and $k=2$ we obtain that 
\begin{align}\label{moment_bound000}
\E[\tilde{\be}(w)^2] 
&= \E\left[ \left(\frac{1}{4L + 2 \rho \al \sum_{j \in \B'} \|\tilde{\nabla} f_j(w, \D_\beta^j)\| /{B'} }\right)^2\right]\nonumber\\
&\leq \frac{\sigma_b^2\frac{1}{(4L)^2} + \mu_b^2 ( \frac{\mu_b}{(\sigma_b^2 + \mu_b(\mu_b+4L))})^2}{\sigma_b^2 + \mu_b^2}
\end{align}
where $\mu_b$ and $\sigma_b^2$ are the mean and variance of random variable
$X=2 \rho \al \frac{1}{B'} \sum_{j \in \B'} \|\tilde{\nabla}f_j(w,\D_\beta^j)\|$. Now replace $\sigma_b^2 + \mu_b(\mu_b+4L)$ by its lower bound $ \mu_b(\mu_b+4L)$ and simplify the terms to obtain
\begin{equation}\label{moment_bound1}
\E[\tilde{\be}(w)^2] \leq \frac{\sigma_b^2/(4L)^2 + \mu_b^2 / (\mu_b+4L)^2}{\sigma_b^2 + \mu_b^2}.   
\end{equation}
Now recall the result in \eqref{D_j_beta_bound} use the fact that the batch size $ D_\beta$ is larger than
$$  D_\beta \geq \ceil[\bigg]{\left(\frac{2 \rho \al \tsigma}{L}\right)^2}$$
to write that
\begin{equation}\label{bounds_on_mu_b}
2 \rho \al \E_{i \sim p} \|\nabla f_i(w)\| - L \leq \mu_b \leq 2 \rho \al \E_{i \sim p} \|\nabla f_i(w)\|	 + L
\end{equation}
Now based on the definition $L(w)  = 4L + 2 \rho \al \E_{i \sim p} \|\nabla f_i(w)\|$ and the first inequality in \eqref{bounds_on_mu_b} we can show that
\begin{equation}
\mu_b + 5L \geq L(w).
\end{equation}
Therefore, using \eqref{moment_bound1}, we have 
\begin{align}
L(w)^2 \E[\tilde{\be}(w)^2] & \leq \frac{\sigma_b^2 (\mu_b+5L)^2/(4L)^2 + \mu_b^2 (\mu_b+5L)^2/ (\mu_b+4L)^2}{\sigma_b^2 + \mu_b^2 } \nonumber \\
& \leq \frac{\mu_b^2 ((5/4)^2+ 2\sigma_b^2/(4L)^2) + 2 (5/4)^2 \sigma_b^2}{\sigma_b^2 +  \mu_b^2} \label{moment_bound2}
\end{align}
where for the last inequality we used the fact that $(\mu_b+5L)^2 \leq 2 \mu_b^2 + 2 (5L)^2$. 
Now considering \eqref{moment_bound2}, to prove the second result in \eqref{first_second_moment} we only need to show that
\begin{equation}\label{goal_var_batch}
2\sigma_b^2/(4L)^2 \leq (5/4)^2.    
\end{equation}
Note that 
\begin{align}
\sigma_b^2 & = \frac{(2 \rho \al)^2}{B'} \text{Var}\left ( \|\tilde{\nabla}f_j(w,\D_\beta^j)\| \right ) = \frac{(2 \rho \al)^2}{B'} \left ( \E \left [ \|\tilde{\nabla}f_j(w,\D_\beta^j)\| ^2 \right] - (\E_{i \sim p} \|\tilde{\nabla}f_j(w,\D_\beta^j)\|)^2\right ) \nonumber \\
& = \frac{(2 \rho \al)^2}{B'} \left ( \text{Var} \left (\tilde{\nabla}f_j(w,\D_\beta^j) \right ) + \|\E \tilde{\nabla}f_j(w,\D_\beta^j) \|^2 - (\E_{i \sim p} \|\tilde{\nabla}f_j(w,\D_\beta^j)\|)^2 \right) \nonumber \\
& \leq \frac{(2 \rho \al)^2}{B'} \left ( \sigma^2 + \frac{\tsigma^2}{|\D_\beta|} + \|\E \tilde{\nabla}f_j(w,\D_\beta^j) \|^2 - (\E_{i \sim p} \|\tilde{\nabla}f_j(w,\D_\beta^j)\|)^2 \right) \label{moment_bound_ineq1}
\end{align}
where the last inequality follows from the law of total variance which states 
\begin{equation}
\text{Var}(Y) = \E \left [ \text{Var}(Y|X) \right] + \text{Var}\left (\E[Y|X] \right )
\end{equation}
for any two random variables $X$ and $Y$ (here $X= \nabla f_j(w)$ and $Y=\tilde{\nabla}f_j(w,\D_\beta^j)$).
Now, using the fact that $|\E[X]| \leq \E[|X|]$ for any random variable $X$, we obtain the following result from \eqref{moment_bound_ineq1}
 \begin{equation}\label{moment_bound_ineq2}	
 \sigma_b^2 \leq \frac{(2 \rho \al)^2}{B'} \left(\sigma^2 + \frac{\tsigma^2}{|\D_\beta|}\right). 
 \end{equation} 
Finally, plugging \eqref{moment_bound_ineq2} in \eqref{moment_bound2} and using the assumption \eqref{size_B'} on size of $B', D_\beta$ completes the proof and the second result in \eqref{first_second_moment} follows. 

To prove the first result in \eqref{first_second_moment} which is a bound on the first moment of $\tilde{\beta}(w)$, note that, using Jensen's inequality we know that $\E[1/X]\geq 1/\E[X]$ and hence by replacing $X$ with $\tilde{L}(w)$ which is defined in \eqref{el_tilde_def} and can be written as  $\tilde{L}(w):=1/\tilde{\be}(w)$ we can show that
\begin{equation}
\E[\tilde{\be}(w)] = \E[\frac{1}{\tilde{L}(w)}] \geq \frac{1}{\E[\tilde{L}(w)]} = \frac{1}{4L+ \mu_b},
\end{equation}
where $\mu_b$ is the mean of $2 \rho \al \frac{1}{B'} \sum_{j \in \B'} \|\tilde{\nabla}f_j(w,\D_\beta^j)\|$.
Now by using this result and the upper bound for $\mu_b$ in \eqref{bounds_on_mu_b} we obtain that
\begin{equation}
\E[\tilde{\be}(w)] \geq \frac{1}{5L + 2 \rho \al \E_{i \sim p} \|\nabla f_i(w)\|} .
\end{equation}
As $L(w)  = 4L + 2 \rho \al \E_{i \sim p} \|\nabla f_i(w)\|$ we can show that
\begin{equation}
\E[\tilde{\be}(w)]  \geq \frac{1}{L + L(w)} \geq \frac{1}{L(w)/4 + L(w)} =\frac{4/5}{L(w)}    
\end{equation}
and the first claim in \eqref{first_second_moment} follows. 
\section{Proof of Lemma \ref{lemma:moments}}\label{lemma:moments_proof}
Note that
\begin{align*}
& \E_{\D_{in},\D_o}[\tilde{\nabla} f_i (w_k-\al \tilde{\nabla} f_i(w_k,\D_{in}^i),\D_{o}^i) \mid \F_k] \\
& = \E_{\D_{in}} \left [ \nabla f_i\!\left(w_k-\al \tilde{\nabla} f_i(w_k,\D_{in}^i) \right) \mid \F_k \right ] \\
& = \E[ \nabla f_{i} \left  (w_k - \alpha \nabla f_{i}(w_k) \right ) \mid \F_k] + \E_{\D_{in}} \left [ \nabla f_i\!\left(w_k-\al \tilde{\nabla} f_i(w_k,\D_{in}^i) \right) - \nabla f_{i} \left  (w_k - \alpha \nabla f_{i}(w_k) \right ) \mid \F_k \right] \\
& = \E [ \nabla f_{i} \left  (w_k - \alpha \nabla f_{i}(w_k) \right ) \mid \F_k] + e_{i,k}
\end{align*}
where 
\begin{equation*}
e_{i,k} = \E_{\D_{in}} \left [ \nabla f_i\!\left(w_k-\al \tilde{\nabla} f_i(w_k,\D_{in}^i) \right) - \nabla f_{i} \left  (w_k - \alpha \nabla f_{i}(w_k) \right )  \mid \F_k \right ]
\end{equation*}
and its norm is bounded by
\begin{align}
\|e_{i,k} \| & \leq \E_{\D_{in}^i} \left [ \left \|\nabla f_i\!\left(w_k-\al \tilde{\nabla} f_i(w_k,\D_{in}^i) \right) - \nabla f_{i} \left  (w_k - \alpha \nabla f_{i}(w_k) \right ) \right \| \mid \F_k \right ] \nonumber \\
& \leq  \al L  \E_{\D_{in}^i} \left [ \left \| \tilde{\nabla} f_i(w_k,\D_{in}^i) - \nabla f_{i}(w_k) \right \| \mid \F_k \right ] \label{lemma_moment_ineq1} \\
& \leq \al L \frac{\tsigma}{\sqrt{D_{in}}}
\end{align}
where \eqref{lemma_moment_ineq1} follows from the Lipschitz property of gradient (Assumption \ref{asm_smooth} and \eqref{smooth_2:a}), and the last line is obtained using Assumption \ref{asm_bounded_var_i}.
To bound the second moment, note that
\begin{align}
& \E_{\D_{in},\D_o} \!\left [ \|\tilde{\nabla} f_i (w_k\!-\!\al \tilde{\nabla} f_i(w_k,\D_{in}^i),\D_{o}^i)\|^2  \mid \F_k \right] \nonumber 	\\
& = \E_{\D_{in}^i} \left [ \| \nabla f_{i} (w_k - \alpha \tilde{\nabla} f_i(w_k,\D_{in}^i) )\|^2 + \frac{\tsigma^2}{D_o} \mid \F_k \right ] \nonumber \\
& \leq (1+ \frac{1}{\phi}) \|\nabla f_{i} (w_k - \alpha \nabla f_{i}(w_k)\|^2 \nonumber \\
& + (1+\phi) \E_{\D_{in}^i} \left [ \| \nabla f_{i} (w_k - \alpha \tilde{\nabla} f_i(w_k,\D_{in}^i) ) - \nabla f_{i}(w_k - \alpha \nabla f_{i}(w_k)\|^2 \mid \F_k \right ] + \frac{\tsigma^2}{D_o} \label{lemma_moment_ineq2} \\
& \leq (1+ \frac{1}{\phi}) \|\nabla f_{i} (w_k - \alpha \nabla f_{i}(w_k)\|^2 + (1+\phi) \al^2 L^2 \frac{\tsigma^2}{D_{in}} + \frac{\tsigma^2}{D_o} \label{lemma_moment_ineq3}
\end{align}
where \eqref{lemma_moment_ineq2} follows from the inequality $(a+b)^2 \leq (1+1/\phi) a^2 + (1+\phi) b^2$ and  \eqref{lemma_moment_ineq3} is obtained similar to \eqref{lemma_moment_ineq1}.
\section{Proof of Theorem \ref{Thm_F_hatF}} \label{proof_Thm_F_hatF}
First, note that
\begin{align}\label{grad_hatF}
\nabla \hat{F}(w) =\E_{i \sim p}\left[\E_{\D_{test}^{i}} \left[ (I - \alpha \tilde{\nabla}^2 f_i(w,\D_{test}^i)) \nabla f_i(w - \al \tilde{\nabla} f_i(w,\D_{test}^i))\right] \right]	
\end{align}
Next, using Assumption \ref{asm_bounded_var_i}, we have
\begin{align}\label{hessian_hatF}
I - \alpha \tilde{\nabla}^2 f_i(w,\D_{test}^i)	= I - \alpha \nabla^2 f_i(w) + e_{H,i}
\end{align}
where 
\begin{equation}
\E_{\D_{test}^{i}}[e_{H,i}] = 0, \quad \E_{\D_{test}^i}[\|e_{H,i}\|^2] \leq \frac{\alpha^2 \sigma_H^2}{D_{test}}. 	
\end{equation}
In addition,
\begin{align}\label{gradient_hatF}
\nabla f_i(w - \al \tilde{\nabla} f_i(w,\D_{test}^i))	= \nabla f_i(w - \al \nabla f_i(w)) + e_{G,i}
\end{align}
where 
\begin{equation*}
e_{G,i} =  \nabla f_i(w - \al \tilde{\nabla} f_i(w,\D_{test}^i)) - \nabla f_{i} \left  (w - \alpha \nabla f_{i}(w_k) \right )
\end{equation*}
and the expectation of its norm squared is bounded by
\begin{align}
\E_{\D_{test}^{i}}[\|e_{G,i} \|^2] & \leq \E_{\D_{test}^i} \left [ \left \|\nabla f_i\!\left(w-\al \tilde{\nabla} f_i(w,\D_{test}^i) \right) - \nabla f_{i} \left  (w - \alpha \nabla f_{i}(w) \right ) \right \|^2 \right ] \nonumber \\
& \leq  \al^2 L^2  \E_{\D_{test}^i} \left [ \left \| \tilde{\nabla} f_i(w,\D_{test}^i) - \nabla f_{i}(w) \right \|^2 \right ] \label{Gradient_hatF_ineq1} \\
& \leq \al^2 L^2 \frac{\tsigma^2}{D_{test}}
\end{align}
where \eqref{Gradient_hatF_ineq1} follows from the Lipschitz property of gradient (Assumption \ref{asm_smooth} and \eqref{smooth_2:a}), and the last line is obtained using Assumption \ref{asm_bounded_var_i}. Now plugging \eqref{hessian_hatF} and \eqref{gradient_hatF} in \eqref{grad_hatF} implies
\begin{align}
\nabla \hat{F}(w) &= \E_{i \sim p}\left[\E_{\D_{test}^i} \left[  (I - \alpha \nabla^2 f_i(w) + e_{H,i}) (\nabla f_i(w - \al \nabla f_i(w)) + e_{G,i})\right] \right]	\\
&= \E_{i \sim p}\left[ (I - \alpha \nabla^2 f_i(w))\nabla f_i(w - \al \nabla f_i(w)) \right ] \nonumber \\
& + \E_{i \sim p}\left[ (I - \alpha \nabla^2 f_i(w)) \E_{\D_{test}^i}[ e_{G,i}] + \nabla f_i(w - \al \nabla f_i(w)) \E_{\D_{test}^i}[  e_{H,i}] \right] \nonumber \\
& + 	\E_{i \sim p}\left[\E_{\D_{test}^i} [e_{G,i} e_{H,i}] \right ].
\end{align}
Using $\nabla F(w) = E_{i \sim p}\left[ (I - \alpha \nabla^2 f_i(w))\nabla f_i(w - \al \nabla f_i(w)) \right ]$ along with $\E_{\D_{test}^i}[e_{H,i}] = 0$ yields
\begin{align}
\nabla \hat{F}(w) &=  \nabla F(w) + \E_{i \sim p}\left[ (I - \alpha \nabla^2 f_i(w)) \E_{\D_{test}^i}[ e_{G,i}]\right] + \E_{i \sim p}\left[\E_{\D_{test}^i} [e_{G,i} e_{H,i}] \right ].
\end{align}
As a result, using the fact that $\|I - \alpha \nabla^2 f_i(w)\| \leq 1 +\alpha L$ along with Cauchy-Schwarz inequality implies
\begin{align}
\|\nabla \hat{F}(w) - \nabla F(w)\| & \leq  (1+\alpha L) \E_{i \sim p} \left [\E_{\D_{test}^i}[ \|e_{G,i}\|] \right ] + \E_{i \sim p} \left [\sqrt{\E_{\D_{test}^i}[\|e_{H,i}\|^2] \E_{\D_{test}^i}[\|e_{G,i}\|^2]} ~ \right ] \nonumber \\
& \leq   (1+\alpha L) \alpha L \frac{\tsigma}{\sqrt{D_{test}}} + \alpha^2 L \frac{\sigma_H \tilde{\sigma}}{D_{test}} \nonumber \\
& \leq 2 \alpha L \frac{\tsigma}{\sqrt{D_{test}}} + \alpha^2 L \frac{\sigma_H \tilde{\sigma}}{D_{test}}
\end{align}
where the last inequality follows from $\alpha \leq \frac{1}{L}$.
\section{Proof of Theorem \ref{Thm_SGD_general_simple} (General Version)} \label{Thm_SGD_general_proof}
\begin{theorem}\label{Thm_SGD_general} 
Consider the objective function $F$ defined in \eqref{main_prob} for the case that $\al \in (0,\frac{1}{6L}]$. Suppose that the conditions in Assumptions~\ref{ass:boundedness}-\ref{asm_bounded_var_i} are satisfied, and recall the definitions $L:=\max L_i$ and $\rho:=\max \rho_i$. Consider running MAML with batch sizes satisfying the conditions $D_h \geq \ceil{2 \al^2 \sigma_H^2}$ and $B \geq 20$.
Let $\beta_k = \tilde{\beta}(w_k)/12$
where $\tilde{\beta}(w)$ is given in defined in \eqref{beta_update}.
Then, for any $\eps >0$, MAML finds a solution $w_\eps$ such that
\begin{equation}\label{result_MAML_general}
 \E[ \| \nabla F(w_\eps) \|] \leq \max\left\{\sqrt{61\left(1 + \frac{\rho \alpha}{L} \sigma\right) \left( \frac{\sigma^2}{B} + \frac{\tsigma^2}{B D_o} + \frac{\tsigma^2}{D_{in}} \right)} , \frac{61\rho \alpha}{L}  \left( \frac{\sigma^2}{B} + \frac{\tsigma^2}{B D_o} + \frac{\tsigma^2}{D_{in}} \right), \eps\right\}	
\end{equation}
after at most running for 
\begin{equation}
\bigO(1) \Delta \min \left \{ \frac{L + \rho \al (\sigma + \eps) }{\eps^2},
\frac{LB}{\sigma^2} + \frac{L(B D_o+D_{in})}{\tsigma^2}
 \right \} 
\end{equation}
iterations, where $\Delta:= (F(w_0) - \min_{w \in \R^d} F(w) )$.
\end{theorem}
\begin{remark}
It is worth noting that the condition $B \geq 20$ can be dropped, i.e., $B$ can be any positive integer, at the cost of 	decreasing the ratio $\beta_k / \tilde{\beta}(w_k)$.
\end{remark}
\begin{proof}
To simplify the notation, we denote $L(w_k)$ {(defined in Lemma \ref{lem_smooth_F})} by $L_k$. Also, let $\mathcal{F}_k$ be the information up to iteration $k$. Note that, conditioning on $\mathcal{F}_k$, the iterate $w_k$, and hence, $F(w_k)$ and $\nabla{F}(w_k)$, are not random variables anymore {, but $\B_k$ and $\D_{in}^i$ used for computing $w_{k+1}^i$ for any $i \in \B_k$ are yet random.}
{In a nutshell, the idea behind this proof (and in fact the other results as well) is to bound the first and second moment of the gradient estimate used in update of MAML by approximating its difference from an unbiased estimator. Next, we apply the descent inequality \eqref{descent_ineq} in Corollary \ref{bergman_bound} to obtain the desired result. More formally, }let
\begin{equation*}
G_i(w_k) := \left(I-\al \tilde{\nabla}^2 f_i(w_{k},D_{h}^i)\right)\ \! \tilde{\nabla} f_i\!\left(w_k-\al \tilde{\nabla} f_i(w_k,\D_{in}^i),\D_{o}^i\right)
\end{equation*}
First, we characterize the first and second moment of $G_i(w_k)$ conditioning on $\F_k$\footnote{we suppress the conditioning on $\F_k$ to simplify the notation} . Note that, since $\D_{in}^i, \D_o^i,$ and $\D_h^i$ are drawn independently, we have
\begin{align}\label{proof_000_001}
\E[G_i(w_k)] 
& = \E_{i \sim p} \left [ \E_{\D_h^i} \left [ I-\al \tilde{\nabla}^2 f_i(w_{k},D_{h}^i) \right ] \E_{\D_o^i, \D_{in}^i} \left [ \tilde{\nabla} f_i\!\left(w_k-\al \tilde{\nabla} f_i(w_k,\D_{in}^i),\D_{o}^i\right) \right ] \right] \nonumber \\
& = \E_{i \sim p} \left [ \left ( I - \al \nabla^2 f_i(w_k) \right ) \left ( \nabla f_{i} \left  (w_k - \alpha \nabla f_{i}(w_k) \right )  + e_{i,k} \right ) \right]
\end{align}
where $e_{i,k}$ as defined in Lemma \eqref{lemma:moments} is given by 
$$e_{i,k}:= \E_{\D_{in},\D_o}[\tilde{\nabla} f_i (w_k-\al \tilde{\nabla} f_i(w_k,\D_{in}^i),\D_{o}^i)]-\nabla f_{i} \left  (w_k - \alpha \nabla f_{i}(w_k) \right ). $$
By simplifying the right hand side of \eqref{proof_000_001} we obtain that
\begin{align}\label{proof_000_002}
\E[G_i(w_k)] 
& = \E_{i \sim p} \left [\left ( I - \al \nabla^2 f_i(w_k) \right ) \nabla f_{i} \left  (w_k - \alpha \nabla f_{i}(w_k) \right )  + \left ( I - \al \nabla^2 f_i(w_k) \right ) e_{i,k} \right] \nonumber \\
& = \E_{i \sim p} \left [ \nabla F_i(w_k) + \left ( I - \al \nabla^2 f_i(w_k) \right ) e_{i,k} \right] \nonumber \\
& = \nabla F(w_k) + r_k 
\end{align}
 and $r_k$ is given by $r_k = \E_{i \sim p} \left [( I - \al \nabla^2 f_i(w_k) ) e_{i,k} \right]$. Note that the second equality in \eqref{proof_000_002} due to definition $F_i(w) := f_i(w - \al \nabla{f_i(w)})$. 
 Next, we derive an upper bound on the norm of $r_k$ as
\begin{align}
\|r_k \| & \leq \E_{i \sim p} \left [\| I - \al \nabla^2 f_i(w_k) \| \|e_{i,k}\| \right ] \nonumber \\
& \leq (1+\al L) \al L \frac{\tsigma}{\sqrt{D_{in}}} \label{general_SGD_ineq1} \\
& \leq 0.2 \frac{\tsigma}{\sqrt{D_{in}}}, \label{general_SGD_ineq2}
\end{align}
where \eqref{general_SGD_ineq1} follows from Lemma \eqref{lemma:moments} along with the Lipschitz property of gradient (Assumption \ref{asm_smooth} and \eqref{smooth_2:a}), and the last line is obtained using the fact that $\al L \leq \frac{1}{6}$. Hence, we have 
$$\|\E[G_i(w_k)] \|\leq  \|\nabla F(w_k)\|+0.2 \frac{\tsigma}{\sqrt{D_{in}}} $$
Now, note that this inequality and the fact that $a\leq b + c$ yields $a^2\leq 2b^2 + 2c^2$ for any positive scalars $a,b,c$, imply that
\begin{equation} \label{general_SGD_ineq1n}
\|\E[G_i(w_k)]\|^2 \leq 2 \| \nabla F(w_k) \|^2 + 0.08 \frac{\tsigma^2}{D_{in}}. 
\end{equation}
To bound the variance of $G_i(w_k)$, we  bound its second moment. A similar argument to what we did above implies
\begin{align}\label{proof_000_003}
\E[\|G_i(w_k)\|^2]
& = \E_{i \sim p} \left [ \E_{\D_h^i} \left \| I-\al \tilde{\nabla}^2 f_i(w_{k},D_{h}^i)\right \|^2 \E_{\D_o^i, D_{in}^i} \left \| \tilde{\nabla} f_i\!\left(w_k-\al \tilde{\nabla} f_i(w_k,\D_{in}^i),\D_{o}^i\right) \right \|^2 \right] 
\end{align}
To simplify the right hand side we first use the fact that 
\begin{align}\label{proof_000_004}
\E_{\D_h^i} \left \| I-\al \tilde{\nabla}^2 f_i(w_{k},D_{h}^i)\right \|^2
&= Var\left[ I-\al \tilde{\nabla}^2 f_i(w_{k},D_{h}^i) \right] + \|I - \al \nabla^2 f_i(w_k)\|^2
\nonumber\\
&= \al^2 Var\left[ \tilde{\nabla}^2 f_i(w_{k},D_{h}^i) \right] + \|I - \al \nabla^2 f_i(w_k)\|^2
\nonumber\\
&\leq \frac{\al^2 \sigma_H^2}{D_h}+ \|I - \al \nabla^2 f_i(w_k)\|^2
\end{align}
where the last inequality follows from Assumption \ref{asm_bounded_var_i}. Substitute the upper bound in \eqref{proof_000_004} into \eqref{proof_000_003} to obtain 
\begin{align}\label{proof_000_005}
\E[\|G_i(w_k)\|^2]
 \leq \E_{i \sim p} \left [ \left ( \|I - \al \nabla^2 f_i(w_k)\|^2 + \frac{\al^2 \sigma_H^2}{D_h} \right ) \E_{\D_o^i, D_{in}^i} \left \| \tilde{\nabla} f_i\!\left(w_k-\al \tilde{\nabla} f_i(w_k,\D_{in}^i),\D_{o}^i\right) \right \|^2 \right]
\end{align}
Note that using the fact that $\|I - \al \nabla^2 f_i(w_k)\| \leq 1+ \al L$ and the assumption that  $\al L \leq \frac{1}{6}$ we can show that $\|I - \al \nabla^2 f_i(w_k)\| \leq 7/6$. Further, we know that $D_h \geq 2 \al^2 \sigma_H^2$ which implies that ${\al^2 \sigma_H^2}/{D_h} \leq 1/2$. By combining these two bounds we can show that
\begin{equation}\label{proof_000_006}
 \|I - \al \nabla^2 f_i(w_k)\|^2 + \frac{\al^2 \sigma_H^2}{D_h} \leq 2 
 \end{equation}
 As a result of \eqref{proof_000_006}, we can simplify the right hand side of \eqref{proof_000_005} to
\begin{align}\label{general_SGD_ineq3}
\E&[\|G_i(w_k)\|^2] \leq 2 \E_{i \sim p} \left [ \E_{\D_o^i, D_{in}^i} \left \| \tilde{\nabla} f_i\!\left(w_k-\al \tilde{\nabla} f_i(w_k,\D_{in}^i),\D_{o}^i\right) \right \|^2 \right].
\end{align}
Note that, using Lemma \ref{lemma:moments} with $\phi =1$, we have
\begin{align}
\E_{\D_o^i, D_{in}^i} \left \| \tilde{\nabla} f_i\!\left(w_k-\al \tilde{\nabla} f_i(w_k,\D_{in}^i),\D_{o}^i\right) \right \|^2 & \leq 2\|\nabla f_{i} (w_k - \alpha \nabla f_{i}(w_k)\|^2 + 2 \al^2 L^2 \frac{\tsigma^2}{D_{in}} + \frac{\tsigma^2}{D_o} \nonumber \\
&\leq 2 \frac{\|\nabla F_i(w_k)\|^2}{(1-\al L)^2} + 2 \al^2 L^2 \frac{\tsigma^2}{D_{in}} + \frac{\tsigma^2}{D_o} \label{general_SGD_ineq5}
\end{align}
where the last inequality follows from \eqref{gradient_F:b} and the fact that $\|I - \al \nabla^2 f_i(w)\| \geq 1- \al L$. Plugging \eqref{general_SGD_ineq5} in \eqref{general_SGD_ineq3} and using \eqref{bound_F_i_by_F} in Lemma \ref{lemma_bound_by_F} yields
\begin{align}\label{general_SGD_ineq6}
\E[\|G_i(w_k)\|^2] \leq 40 \| \nabla F(w_k)\|^2 + 14 \sigma^2  + \tsigma^2 \left(\frac{2}{D_o} + \frac{1}{6D_{in}}\right). 
\end{align}
Now that we have upper bounds on $\E[\|G_i(w_k)\|] $ and $\E[\|G_i(w_k)\|^2] $, we proceed to prove the main result. According to Corollary \ref{bergman_bound}, we have
\begin{align}
F(w_{k+1}) & \leq  F(w_k) + \nabla{F}(w_k)^\top (w_{k+1}-w_k) + \frac{L_k}{2} \| w_{k+1}-w_k \|^2  \nonumber \\
& =  F(w_k) - \beta_{k} \nabla{F}(w_k)^\top \left ( \frac{1}{B} \sum_{i \in \B_k} G_i(w_k)  \right) + \frac{L_k}{2} \beta_{k}^2 \left \| \frac{1}{B} \sum_{i \in \B_k} G_i(w_k) \right \|^2. \label{general_SGD_ineq8}
\end{align}
By computing the expectation of both sides of \eqref{general_SGD_ineq8} conditioning on $\F_k$, we obtain that
\begin{align}
\E[F(w_{k+1})|\F_k] & \leq  F(w_k) - \E[\beta_{k}|\F_k] \nabla{F}(w_k)^\top \E[G_i(w_k)|\F_k] \nonumber \\
& +\frac{L_k}{2} \E[\beta_{k}^2|\F_k] \left ( \|\E[G_i(w_k)|\F_k]\|^2  + \frac{1}{B} \E[\|G_i(w_k)\|^2| \F_k] \right) \nonumber
\end{align}
where we used the fact that batches $\B_k$ and $\B'_k$ are independently drawn. Now, using the expression $$\E[G_i(w_k)|\F_k] = \nabla F(w_k) + r_k$$ in \eqref{proof_000_002}
along with \eqref{general_SGD_ineq1n} and \eqref{general_SGD_ineq6} we can write that
\begin{align}\label{general_SGD_ineq9}
\E & [F(w_{k+1}) |\F_k] \leq  F(w_k) - \|\nabla{F}(w_k)\|^2 \left ( \E[\beta_{k}|\F_k] - \frac{L_k}{2} \E[\beta_{k}^2|\F_k] \left(2+ \frac{40}{B}\right)\right) \nonumber \\
& + \E[\beta_{k}|\F_k] \|\nabla{F}(w_k)\| \|r_k\| +  \frac{L_k}{2} \E[\beta_{k}^2|\F_k] \left ( \frac{1}{B} \left ( 14 \sigma^2  + \tsigma^2 \left(\frac{2}{D_o} + \frac{0.2}{D_{in}}\right) \right ) + 0.08 \frac{\tsigma^2}{D_{in}} \right ).
\end{align}
Note that, using \eqref{general_SGD_ineq2}, we can show that
\begin{equation*}
\|\nabla{F}(w_k)\| \|r_k\| \leq \frac{\|\nabla{F}(w_k)\|^2}{10} +  10 \|r_k\|^2 \leq \frac{\|\nabla{F}(w_k)\|^2}{10} +  0.4 \frac{\tsigma^2}{D_{in}}.   
\end{equation*}
Plugging this bound in \eqref{general_SGD_ineq9} implies
\begin{align}\label{descent_SGD_condition}
\E&[F(w_{k+1})|\F_k]  \leq  F(w_k) - \|\nabla{F}(w_k)\|^2 \left ( \frac{9}{10} \E[\beta_{k}|\F_k] - \frac{L_k}{2} \E[\beta_{k}^2|\F_k] \left(2+ \frac{40}{B}\right)\right) \nonumber \\
& +  \frac{L_k}{2} \E[\beta_{k}^2|\F_k] \left ( \frac{1}{B} \left ( 14 \sigma^2  + \tsigma^2 \left(\frac{2}{D_o} + \frac{0.2}{D_{in}}\right) \right ) + 0.08 \frac{\tsigma^2}{D_{in}} \right ) + 0.4 \E[\beta_{k}|\F_k] \frac{\tsigma^2}{D_{in}}.
\end{align}
Note that $\beta_k = \tilde{\beta}(w_k)/12$, and hence, by using Lemma \ref{tilde_beta} along with $1/{\tilde{\beta}(w_k)} ,L_k \geq 4L$, we have
\begin{equation*}
\frac{1}{48 L} \geq \E[\beta_{k}|\F_k] \geq \frac{1}{15 L_k}, \quad \frac{L_k}{2} \E[\beta_{k}^2|\F_k] \leq \frac{1}{92 L_k} \leq \frac{1}{368 L}.
\end{equation*}
Plugging these bounds in \eqref{descent_SGD_condition} and using the assumption $B \geq 20$ 
yields
\begin{align}
\E & [F(w_{k+1})|\F_k] \nonumber \\
& \leq  F(w_k) - \frac{1}{100 L_k} \|\nabla{F}(w_k)\|^2  + \frac{1}{368 L B}\left ( 14 \sigma^2  +\tsigma^2 \left(\frac{2}{D_o} + \frac{0.2}{D_{in}}\right) \right ) + \frac{({0.4}/{48} + {0.08}/{368})\tsigma^2}{L D_{in}} \nonumber \\
& \leq  F(w_k) - \frac{1}{100 L_k} \|\nabla{F}(w_k)\|^2  + \frac{14 \sigma^2  + {2 \tsigma^2}/{D_o}}{368 L B}+ \frac{\tsigma^2}{96 L D_{in}},\label{general_SGD_ineq10}
\end{align}
where the last inequality is obtained by taking the ${0.2 \tsigma^2}/{D_{in}}$ from the second term and merging it with the third term. 
Next, note that
\begin{align}\label{thm1_ineq1}
\frac{1}{L_k} \| \nabla{F}(w_k)\|^2 = \frac{\| \nabla{F}(w_k)\|^2}{4L + 2 \rho \al \E_{i \sim p} \|\nabla f_i(w_k)\|} \geq \frac{\| \nabla{F}(w_k)\|^2}{4L + 2 \rho \al \sigma + 2 \rho \al \|\nabla f(w_k)\|} 
\end{align}
where the last inequality follows from \eqref{lem2_ineq4}. Using Lemma \ref{lemma_bound_by_F} along with the fact that $\al \leq \frac{1}{6L}$, implies
\begin{equation}\label{thm1_ineq2}
\|\nabla f(w_k)\| \leq 2 \| \nabla{F}(w_k)\| + \sigma.
\end{equation}
Plugging \eqref{thm1_ineq2} in \eqref{thm1_ineq1} yields
\begin{align}
\frac{1}{L_k} \| \nabla{F}(w_k)\|^2 & \geq \frac{\| \nabla{F}(w_k)\|^2}{4L + 4 \rho \al \sigma + 4 \rho \al \| \nabla{F}(w_k)\|} \label{thm1_ineq3}
\end{align}
Now, plugging \eqref{thm1_ineq3} in \eqref{general_SGD_ineq10} and taking expectation from both sides with respect to $\F_k$ along with using tower rule implies
\begin{align}\label{general_SGD_ineq11_1}
\E[F(w_{k+1})] \leq  \E[F(w_k)] - \frac{1}{100} \E \left [ \frac{\| \nabla{F}(w_k)\|^2}{4L + 4 \rho \al \sigma + 4 \rho \al \| \nabla{F}(w_k)\|} \right ] + \frac{14 \sigma^2  + {2 \tsigma^2}/{D_o}}{368 L B}+ \frac{\tsigma^2}{96 L D_{in}}.   
\end{align}
Note that, by Cauchy-Schwartz inequality, we have $E[X] E[Y] \geq \E[\sqrt{XY}]^2$ for nonnegative random variables $X$ and $Y$. Choosing $X = \| \nabla{F}(w_k)\|^2/ \left (4L + 4 \rho \al \sigma + 4 \rho \al \| \nabla{F}(w_k)\| \right )$ and $Y=4L + 4 \rho \al \sigma + 4 \rho \al \| \nabla{F}(w_k)\|$, we obtain
\begin{align}\label{general_SGD_ineq11_11}
\E & \left [ \frac{\| \nabla{F}(w_k)\|^2}{4L + 4 \rho \al \sigma + 4 \rho \al \| \nabla{F}(w_k)\|} \right ] \geq \frac{\E[ \| \nabla{F}(w_k)\|]^2}{\E \left [4L + 4 \rho \al \sigma + 4 \rho \al \| \nabla{F}(w_k)\| \right ]} 
\end{align}
As a result, we have 
\begin{align}\label{general_SGD_ineq11_2}
\E & \left [ \frac{\| \nabla{F}(w_k)\|^2}{4L + 4 \rho \al \sigma + 4 \rho \al \| \nabla{F}(w_k)\|} \right ] \geq \frac{\E[ \| \nabla{F}(w_k)\|]^2}{4L + 4 \rho \al \sigma + 4 \rho \al \E[\| \nabla{F}(w_k)\| ]} \nonumber \\
& \geq \frac{\E[\| \nabla{F}(w_k)\|]^2}{2 \max \{ 4L + 4\rho \al \sigma, 4\rho \al \E[\| \nabla{F}(w_k)\|] \} }  = \min \left \{  \frac{\E[\| \nabla{F}(w_k)\|]^2}{8L + 8\rho \al \sigma }, \frac{\E[\| \nabla{F}(w_k)\|]}{8 \rho \al}   \right \}.	
\end{align}
Plugging \eqref{general_SGD_ineq11_2} in \eqref{general_SGD_ineq11_1} implies
\begin{align}\label{general_SGD_ineq11}
\E[F(w_{k+1})] \leq  \E[F(w_k)] - \frac{1}{800} \min \left \{  \frac{\E[ \| \nabla{F}(w_k)\|]^2}{L + \rho \al \sigma }, \frac{\E[\| \nabla{F}(w_k)\|]}{\rho \al}   \right \} + \frac{14 \sigma^2  + {2 \tsigma^2}/{D_o}}{368 L B}+ \frac{\tsigma^2}{96 L D_{in}}.   
\end{align}
Assume \eqref{result_MAML_general} does not hold at iteration $k$. Then, we have
\begin{equation*}
\E[\| \nabla F(w_k) \|] \geq \max\{\sqrt{(1 + \frac{\rho \alpha}{L} \sigma)\gamma_1} , \frac{\rho \alpha}{L} \gamma_1\}	
\end{equation*}
with $\gamma_1$ given by
\begin{equation}
\gamma_1 = 61 \left( \frac{\sigma^2}{B} + \frac{\tsigma^2}{B D_o} + \frac{\tsigma^2}{D_{in}} \right).	
\end{equation}
This implies 
\begin{equation*}
\frac{1}{1600} \min \left \{  \frac{\E[ \| \nabla{F}(w_k)\|]^2}{L + \rho \al \sigma }, \frac{\E[\| \nabla{F}(w_k)\|]}{\rho \al}   \right \} \geq  \frac {\gamma_1}{1600 L} \geq \frac{14 \sigma^2  + { 2\tsigma^2}/{D_o}}{368 L B}+ \frac{ \tsigma^2}{96 L D_{in}},	
\end{equation*}
and hence, using \eqref{general_SGD_ineq11}, we obtain
\begin{align*}
\E[F(x_{w+1})] \leq \E[F(w_k)] - \frac{1}{1600} \min \left \{  \frac{\E[ \| \nabla{F}(w_k)\|]^2}{L + \rho \al \sigma }, \frac{\E[\| \nabla{F}(w_k)\|]}{\rho \al}   \right \} \leq  \E[F(w_k)] - \frac {\gamma_1}{1600 L}.
\end{align*}
Based on the assumption that  \eqref{result_MAML_general} does not hold at iteration $k$
we also know that $\E[\| \nabla{F}(w_k)\|] \geq \eps$ which implies that
\begin{align}\label{telescope}
\E[F(w_{k+1})] \leq  \E[F(w_k)] - \frac{1}{1600} \min \left \{ \frac{\eps ^2}{L+ \rho \al \sigma }, \frac{\eps}{\rho \al} \right \} \leq \E[F(w_k)] - \frac{1}{1600} \frac{\eps ^2}{L+ \rho \al (\sigma + \eps)}.
\end{align}
This result shows that if the condition in \eqref{result_MAML_general} is not satisfied the objective function value decreases by a constant value in expectation. If we assume that for all iterations $0,\dots,T-1$ this condition does not hold then by summing both sides of \eqref{telescope} from $0$ to $T-1$ we obtain that
\begin{align}\label{telescope2}
\sum_{k=0}^{T-1}\E[F(w_{k+1})]  \leq \sum_{k=0}^{T-1}\E[F(w_k)] - \sum_{k=0}^{T-1}\frac{1}{1600} \frac{\eps ^2}{L+ \rho \al (\sigma + \eps)}.
\end{align}
which implies that 
\begin{align}\label{telescope3}
\E[F(w_{T})]  \leq \E[F(w_0)] -\frac{T}{1600} \frac{\eps ^2}{L+ \rho \al (\sigma + \eps)}.
\end{align}
and hence
\begin{align}\label{telescope4}
 T & \leq (\E[F(w_0)] -\E[F(w_{T})]) {1600} \frac{L+ \rho \al (\sigma + \eps)}{\eps ^2}\nonumber\\
 &\leq (F(w_0) -F(w^*)) {1600} \frac{L+ \rho \al (\sigma + \eps)}{\eps ^2}
\end{align}
This argument shows that if the condition in \eqref{result_MAML_general} is not satisfied for all $k$ form $0$ to $T-1$, then the time $T$ can not be larger than $(F(w_0) -F(w^*)) {1600} \frac{L+ \rho \al (\sigma + \eps)}{\eps ^2}$. Hence, after $(F(w_0) -F(w^*)) {1600} \frac{L+ \rho \al (\sigma + \eps)}{\eps ^2}$ iterations at least one of the iterates generated by MAML satisfies the condition in \eqref{result_MAML_general}, and the proof is complete. 
\end{proof}
\section{Proof of Theorem \ref{First_order_MAML_det_simple} (General Version)}\label{First_order_MAML_det_proof}
\begin{theorem}\label{First_order_MAML_det}
Consider the objective function $F$ defined in \eqref{main_prob} for the case that $\al \in (0,\frac{1}{10L}]$. Suppose that the conditions in Assumptions~\ref{ass:boundedness}-\ref{asm_bounded_var_i} are satisfied, and recall the definitions $L:=\max L_i$ and $\rho:=\max \rho_i$. Consider running FO-MAML with batch sizes satisfying the conditions $D_h \geq \ceil{2 \al^2 \sigma_H^2}$ and $B \geq 20$.
Let $\beta_k = \tilde{\beta}(w_k)/18$
where $\tilde{\beta}(w)$ is given in defined in \eqref{beta_update}.
Then, for any $\eps >0$, first order MAML finds a solution $w_\eps$ such that
{\small \begin{align}\label{result_FOMAML_general}
& \E [ \| \nabla F(w_\eps) \|] \leq \nonumber \\
 &\max\left\{\sqrt{14\left(1 + \frac{\rho \alpha}{L} \sigma\right) \left ( \sigma^2 (1/B+20 \alpha^2 L^2) + \frac{\tsigma^2}{B D_o} + \frac{\tsigma^2}{D_{in}}\right )	} , \frac{14\rho \alpha}{L}  \left ( \sigma^2 (\frac{1}{B}+20 \alpha^2 L^2) + \frac{\tsigma^2}{B D_o} + \frac{\tsigma^2}{D_{in}}\right )	, \eps \right\}	
 \end{align}}
after at most running for 
\begin{equation}
\bigO(1) \Delta \min \left \{ \frac{L + \rho \al (\sigma + \eps) }{\eps^2}, \frac{L}{\sigma^2 (1/B+20 \alpha^2 L^2) }+\frac{L(BD_o+D_{in})}{\tsigma^2} \right \}
\end{equation}
iterations, where $\Delta:= (F(w_0) - \min_{w \in \R^d} F(w) )$.
\end{theorem}
\begin{proof}
First, note that the update of the first-order approximation of MAML can be written as $w_{k+1}=w_k- \frac{\beta_{k}}{B} \sum_{i \in \B_k} G_i(w_k)$, where
\begin{equation*}
G_i(w_k) := \tilde{\nabla} f_i\!\left(w_k-\al \tilde{\nabla} f_i(w_k,\D_{in}^i),\D_{o}^i\right).
\end{equation*}
To analyze this update,  similar to the proof of Theorem \ref{Thm_SGD_general}, we first characterize the first and second moment of the descent direction $G_i(w)$ conditioning on $\F_k$. Using the definition
$$
e_{i,k} = \nabla f_{i} \left  (w_k - \alpha \nabla f_{i}(w_k) \right ) -\E_{\D_{in},\D_o}[\tilde{\nabla} f_i (w_k-\al \tilde{\nabla} f_i(w_k,\D_{in}^i),\D_{o}^i)] 
$$
we can write that 
\begin{align}\label{proof_001_001}
\E[G_i(w_k)] = \E_{i \sim p} \left [ \nabla f_{i} \left  (w_k - \alpha \nabla f_{i}(w_k) \right ) +e_{i,k} \right]. 
\end{align}
Further, based on the definition of $F_i$ and the fact that its gradient is given by $\nabla F_i (w)= ( I - \al \nabla^2 f_i(w))\nabla f_{i} \left  (w - \alpha \nabla f_{i}(w) \right ) $, we can rewrite the right hand side of \eqref{proof_001_001} as 
\begin{align}\label{proof_001_002}
\E[G_i(w_k)] 
= \E_{i \sim p} \left [ \left ( I - \al \nabla^2 f_i(w_k) \right )^{-1} \nabla F_i(w_k) + e_{i,k} \right] 
\end{align}
Now add and subtract $\nabla F_i(w_k)$ to the right hand side of \eqref{proof_001_002} and use the fact that $\E_{i \sim p}[ \nabla F_i(w_k)] = \nabla F(w_k)$ to obtain 
\begin{align}\label{proof_001_003}
\E[G_i(w_k)] 
&= \E_{i \sim p} \left [ \left ( I - \al \nabla^2 f_i(w_k) \right )^{-1} \nabla F_i(w_k) - \nabla F_i(w_k) +\nabla F_i(w_k)+ e_{i,k} \right] \nonumber\\
&=  \nabla F(w_k) + \E_{i \sim p} \left [ \left ( I - \al \nabla^2 f_i(w_k) \right )^{-1} \nabla F_i(w_k) - \nabla F_i(w_k)+ e_{i,k} \right] \nonumber\\
&=  \nabla F(w_k) + \E_{i \sim p} \left [ \left(\left ( I - \al \nabla^2 f_i(w_k) \right )^{-1} -I\right) \nabla F_i(w_k) + e_{i,k} \right]
\end{align}
To simplify the expressions let us define $ r_k $ as 
\begin{align}\label{proof_001_004}
 r_k = \E_{i \sim p} \left [ \left ( (I - \al \nabla^2 f_i(w_k))^{-1} - I \right ) \nabla F_i(w_k) +e_{i,k} \right ] 
\end{align}
Using the definition of $r_k$ in \eqref{proof_001_004} we can rewrite \eqref{proof_001_003} as 
\begin{align}\label{proof_001_005}
\E[G_i(w_k)]  = \nabla F(w_k) + r_k 
\end{align}
Now we proceed to simplify the expression for $r_k$. Note that using the expansion 
\begin{equation}\label{inverse_expansion}
(I - \al \nabla^2 f_i(w_k))^{-1} = I + \sum_{j=1}^\infty \al^j (\nabla^2 f_i(w_k))^j .    
\end{equation}
we can rewrite $r_k$ defined in \eqref{proof_001_004} as
\begin{align}\label{proof_001_006}
 r_k = \sum_{j=1}^\infty \al^j  \E_{i \sim p} \left [(\nabla^2 f_i(w_k))^j \nabla F_i(w_k) \right ] + \E_{i \sim p}[e_{i,k}]
\end{align}
Next we derive an upper bound on the norm of $r_k$. The $l_2$ norm of the first term in \eqref{proof_001_006} can be upper bounded by 
\begin{align}\label{proof_001_007}
\left\|\sum_{j=1}^\infty \al^j  \E_{i \sim p} \left [(\nabla^2 f_i(w_k))^j \nabla F_i(w_k) \right ]\right\| 
&\leq \sum_{j=1}^\infty \al^j L^j \E_{i \sim p} \| \nabla F_i(w_k)\| \nonumber\\
&\leq \frac{\al L}{1- \al L} \E_{i \sim p} \| \nabla F_i(w_k)\| \nonumber \\
& \leq 0.22 \|\nabla F(w_k)\| + 2 \alpha L \sigma
\end{align}
where the last inequality follows from Lemma \ref{lemma_bound_by_F} and the fact that $\alpha L \leq 1/10$.  Further, based on the result in Lemma \ref{lemma:moments}  we know that $\|e_{i,k}\|$ for any $i$ is bounded above by $ \frac{\al L\tsigma}{\sqrt{D_{in}}}$. Indeed, when norm of a random variable is bounded above by a constant, norm of its expectation is also upper bounded by that constant. Hence, we can write
\begin{align}\label{proof_001_008}
\|\E_{i \sim p}[e_{i,k}]\|\leq \frac{\al L\tsigma}{\sqrt{D_{in}}}
\end{align}
Using the inequalities in \eqref{proof_001_007} and \eqref{proof_001_008} and the definition of $r_k$ in \eqref{proof_001_006} we can show that $\|r_k\|$ is upper bounded by
\begin{align}\label{thm_FMAML_ineq1}
\|r_k \| \leq  0.22 \|\nabla F(w_k)\| + 2 \alpha L \sigma + 0.1 \frac{\tsigma}{\sqrt{D_{in}}}. 
\end{align}
Hence, by using the inequality $(a+b+c)^2\leq 3a^2+3b^2+3c^2$ we can show that
\begin{equation}\label{thm_FMAML_ineq1.5}
\|r_k\|^2 \leq 	0.15 \| \nabla F(w_k) \|^2 + 12 \alpha^2 L^2 \sigma^2 + 0.03 \frac{\tsigma^2}{D_{in}}
\end{equation}
Considering this result and the expression in \eqref{proof_001_005} we can write
\begin{align}
\|\E[G_i(w_k)]\|^2 \leq 2 \| \nabla F(w_k) \|^2 + 2 \|r_k\|^2 \leq 2.3 \| \nabla F(w_k) \|^2 +  24 \alpha^2 L^2 \sigma^2  +  0.06 \frac{\tsigma^2}{D_{in}}. \label{FMAML_ineq1n} 
\end{align}
Next, we can derive an upper bound on the second moment of $\|G_i(w_k)\|^2$ similar to the way that we derived \eqref{general_SGD_ineq6} in the proof of Theorem \ref{Thm_SGD_general}. More precisely, note that 
\begin{align}\label{general_FMAML_ineq3}
\E&[\|G_i(w_k)\|^2] = \E_{i \sim p} \left [ \E_{\D_o^i, D_{in}^i} \left \| \tilde{\nabla} f_i\!\left(w_k-\al \tilde{\nabla} f_i(w_k,\D_{in}^i),\D_{o}^i\right) \right \|^2 \right].
\end{align}
Using Lemma \ref{lemma:moments} with $\phi =1$, we have
\begin{align}
\E_{\D_o^i, D_{in}^i} \left \| \tilde{\nabla} f_i\!\left(w_k-\al \tilde{\nabla} f_i(w_k,\D_{in}^i),\D_{o}^i\right) \right \|^2 & \leq 2\|\nabla f_{i} (w_k - \alpha \nabla f_{i}(w_k)\|^2 + 2 \al^2 L^2 \frac{\tsigma^2}{D_{in}} + \frac{\tsigma^2}{D_o} \nonumber \\
&\leq 2 \frac{\|\nabla F_i(w_k)\|^2}{(1-\al L)^2} + 2 \al^2 L^2 \frac{\tsigma^2}{D_{in}} + \frac{\tsigma^2}{D_o} \label{general_FMAML_ineq5}
\end{align}
where the last inequality follows from \eqref{gradient_F:b} and the fact that $\|I - \al \nabla^2 f_i(w)\| \geq 1- \al L$. Plugging \eqref{general_FMAML_ineq5} in \eqref{general_FMAML_ineq3} and using \eqref{bound_F_i_by_F} in Lemma \ref{lemma_bound_by_F} yields
\begin{align}
\E[\|G_i(w_k)\|^2] & \leq 20 \| \nabla F(w_k)\|^2 + 7 \sigma^2  + \tsigma^2 (\frac{1}{D_o} + \frac{0.02}{D_{in}}).
\end{align}
Also, using the same argument in deriving \eqref{general_SGD_ineq8}, \eqref{general_SGD_ineq9}, and \eqref{descent_SGD_condition} in the proof of Theorem \ref{Thm_SGD_general}, we obtain
\begin{align}\label{thm_FMAML_ineq2}
& \E[F(w_{k+1})|\F_k] \nonumber \\
& \leq  F(w_k) - \|\nabla{F}(w_k)\|^2 \left ( \E[\beta_{k}|\F_k] - \frac{L_k}{2} \E[\beta_{k}^2|\F_k] (2.3+ \frac{20}{B})\right) \nonumber \\
& + \E[\beta_{k}|\F_k] \|\nabla{F}(w_k)\| \|r_k\| +  \frac{L_k}{2} \E[\beta_{k}^2|\F_k] \left ( \frac{1}{B} \left ( 7 \sigma^2  + \tsigma^2 (\frac{1}{D_o} + \frac{0.02}{D_{in}}) \right ) + 24 \alpha^2 L^2 \sigma^2  +  0.06 \frac{\tsigma^2}{D_{in}} \right ).
\end{align}
Note that, using \eqref{thm_FMAML_ineq1.5}, we have 
\begin{equation*}
\|\nabla{F}(w_k)\| \|r_k\| \leq \frac{1}{2} \left (\frac{\|\nabla{F}(w_k)\|^2}{2} +  2 \|r_k\|^2 \right ) \leq 0.4{\|\nabla{F}(w_k)\|^2} +  0.03 \frac{\tsigma^2}{D_{in}} + 12 \alpha^2 L^2 \sigma^2.   
\end{equation*}
Plugging this bound in \eqref{thm_FMAML_ineq2} implies
\begin{align*}
\E & [F(w_{k+1})|\F_k] \leq  F(w_k) - \|\nabla{F}(w_k)\|^2 \left ( 0.6 \E[\beta_{k}|\F_k] - \frac{L_k}{2} \E[\beta_{k}^2|\F_k] (2.3+ \frac{20}{B})\right) \\
& +  \frac{L_k}{2} \E[\beta_{k}^2|\F_k] \left ( \frac{1}{B} \left ( 7 \sigma^2  + \tsigma^2 (\frac{1}{D_o} + \frac{0.02}{D_{in}}) \right ) + 24 \alpha^2 L^2 \sigma^2  +  0.06 \frac{\tsigma^2}{D_{in}} \right ) \\
& + \E[\beta_{k}|\F_k] (12 \alpha^2 L^2 \sigma^2 + 0.03 \frac{\tsigma^2}{D_{in}}).
\end{align*}
Using $\beta_k = \tilde{\beta}(w_k)/18$, and with similar analysis as Theorem \ref{Thm_SGD_general}, we obtain
\begin{align*}
\E[F(w_{k+1})|\F_k] & \leq  F(w_k) - \frac{1}{100 L_k} \|\nabla{F}(w_k)\|^2  + \sigma^2(\frac{7}{828LB} + \frac{\alpha^2 L}{6} ) + \frac{{\tsigma^2}/{D_o}}{828 L B}+ \frac{\tsigma^2/D_{in}}{600 L}
\end{align*}
which is similar to \eqref{general_SGD_ineq10}, and the rest of proof follows same as the way that we derived \eqref{thm1_ineq1}- \eqref{telescope4} in the proof of Theorem \ref{Thm_SGD_general}.
\end{proof}

\section{Proof of Theorem \ref{Thm_HF_MAML_simple} (General Version)}\label{Thm_HF_MAML_proof}
\begin{theorem}\label{Thm_HF_MAML}
Consider the objective function $F$ defined in \eqref{main_prob} for the case that $\al \in (0,\frac{1}{6L}]$. Suppose that the conditions in Assumptions~\ref{ass:boundedness}-\ref{asm_bounded_var_i} are satisfied, and recall the definitions $L:=\max L_i$ and $\rho:=\max \rho_i$. Consider running HF-MAML with batch sizes satisfying the conditions $D_h \geq \ceil{36 (\al \rho \tsigma)^2}$ and $B \geq 20$.
Let $\beta_k = \tilde{\beta}(w_k)/25$
where $\tilde{\beta}(w)$ is defined in \eqref{beta_update}. Also, we choose the approximation parameter $\delta_k^i$ in HF-MAML as
\begin{equation*}
\delta_k^i = \frac{1}{6 \rho \alpha \| \tilde{\nabla} f_i(w_k\!-\!\al \tilde{\nabla} f_i(w_k,\D_{in}^i),\D_{o}^i) \|}
\end{equation*}
Then, for any $\eps >0$, HF-MAML finds a solution $w_\eps$ such that
\begin{equation}\label{result_HFMAML_general}
 \E[ \| \nabla F(w_\eps) \|] \leq \max\left\{6\sqrt{(1 + \frac{\rho \alpha}{L} \sigma) \left ( \frac{\sigma^2}{B} + \frac{\tsigma^2}{B D_o} + \frac{\tsigma^2}{D_{in}} \right )},\ \! 36\frac{\rho \alpha}{L} \left ( \frac{\sigma^2}{B} + \frac{\tsigma^2}{B D_o} + \frac{\tsigma^2}{D_{in}} \right),\ \! \eps\right\}	
\end{equation}
after at most running for 
\begin{equation}
\bigO(1) \Delta \min \left \{ \frac{L + \rho \al (\sigma + \eps) }{\eps^2},\ \! \frac{LB}{\sigma^2} + \frac{L(B D_o+D_{in})}{\tsigma^2}\right \} 
\end{equation}
iterations, where $\Delta:= (F(w_0) - \min_{w \in \R^d} F(w) )$.
\end{theorem}
\begin{proof}
Note that the update of the first-order approximation of MAML can be written as $w_{k+1}=w_k- \frac{\beta_{k}}{B} \sum_{i \in \B_k} G_i(w_k)$, where
\begin{equation*}
G_i(w) :=  \tilde{\nabla} f_i\!\left(w_k-\al \tilde{\nabla} f_i(w_k,\D_{in}^i),\D_{o}^i\right)  -\alpha  d_k^i,
\end{equation*}
and $d_k^i$ is given by \eqref{d_HF_MAML}. Similar to previous proofs, we first derive upper bounds on the first and second moment of $G_i(w_k)$. Using the definition
\begin{equation}\label{yechi}
e_{i,k} = \nabla f_{i} \left  (w_k - \alpha \nabla f_{i}(w_k) \right ) -\E_{\D_{in},\D_o}[\tilde{\nabla} f_i (w_k-\al \tilde{\nabla} f_i(w_k,\D_{in}^i),\D_{o}^i)] 
\end{equation}
we can write that
\begin{align}\label{HF_MAML_eq1}
\E[G_i(w)] & = \E_{i \sim p} \left [ \nabla f_i\!\left(w_k-\al \nabla f_i(w_k) \right) \right ] + \E_{i \sim p}[ e_{i,k}] - \alpha ~ \E_
p[ d_k^i]
\end{align}
Next, note that 
{\small
\begin{align}
& \E_{\D_o^i, \D_{in}^i, \D_h^i} [ d_k^i] \nonumber \\
& = \E_{\D_o^i, \D_{in}^i} \left [ \frac{{\nabla} f_i \!\left(w_k\!+\!\delta_k^i \tilde{\nabla} f_i(w_k\!-\!\al \tilde{\nabla} f_i(w_k,\D_{in}^i),\D_{o}^i) \right)\!-\!{\nabla} f_i\!\left(w_k\!-\!\delta_k^i \tilde{\nabla} f_i(w_k\!-\!\al \tilde{\nabla} f_i(w_k,\D_{in}^i),\D_{o}^i) \right)}{2\delta_k^i} \right ] \nonumber \\
& = \E_{\D_o^i, \D_{in}^i} \left [ \nabla^2 f_i(w_k) \tilde{\nabla} f_i(w_k\!-\!\al \tilde{\nabla} f_i(w_k,\D_{in}^i),\D_{o}^i) + \tilde{e}_{k}^i \right ] \label{HF_MAML_eq1.5}
\end{align}}
with
\begin{align}\label{kln}
 \tilde{e}_{k}^i  =
\nabla^2 f_i (w_k)  v - \left[ 
  \frac{\nabla f_i(w_k+\delta_k^i v)-\nabla f_i(w_k-\delta_k^i v)}{2\delta_k^i}\right]
\end{align}
where $v = \tilde{\nabla} f_i(w_k\!-\!\al \tilde{\nabla} f_i(w_k,\D_{in}^i),\D_{o}^i)$. 
Next, by using the definition in \eqref{yechi} we can simplify \eqref{HF_MAML_eq1.5} and write
\begin{equation}\label{HF_MAML_eq2}
\E_{\D_o^i, \D_{in}^i, \D_h^i} [ d_k^i] = \nabla^2 f_i(w_k) \left (  \nabla f_i\!\left(w_k-\al \nabla f_i(w_k) \right) + e_{i,k} \right ) + \E_{\D_o^i, \D_{in}^i} [\tilde{e}_{k}^i ]	
\end{equation} 
Plugging \eqref{HF_MAML_eq2} in \eqref{HF_MAML_eq1}, we obtain
\begin{align}\label{HF_MAML_eq30}
\E[G_i(w)] & = \E_{i \sim p} \left [ \left ( I - \al \nabla^2 f_i(w_k) \right ) \left (  \nabla f_i\!\left(w_k-\al \nabla f_i(w_k) \right) + e_{i,k} \right ) \right] - \alpha ~ \E [\tilde{e}_{k}^i] \nonumber \\
& = \nabla F(w_k) +	\E_{i \sim p} \left [ \left ( I - \al \nabla^2 f_i(w_k) \right ) e_{i,k} \right] - \alpha ~ \E [\tilde{e}_{k}^i]. 
\end{align} 

Now we proceed to bound the norm of each term in the right hand side of \eqref{HF_MAML_eq30}. First, note that according to Lemma \ref{lemma:moments} we know that $\|e_{i,k}\|$ is bounded above by 
\begin{align}
 \|e_{i,k} \| \leq \frac{\al L\tsigma}{\sqrt{D_{in}}}.
\end{align}
Therefore, we can show that
\begin{align} \label{HF_MAML_eq3.5}
\E_{i \sim p} \left [ \left ( I - \al \nabla^2 f_i(w_k) \right ) e_{i,k} \right]
\leq \E_{i \sim p} \left [ \| I - \al \nabla^2 f_i(w_k)\| \| e_{i,k} \|\right]
 \leq (1+ \alpha L) \frac{\al L \tsigma}{\sqrt{D_{in}}}
\end{align}

Next, we derive an upper bound on $\|\tilde{e}_{k}^i\|$. Note that for any vector $v$ can show that
\begin{equation*}
 \left \| \nabla^2 f_i (w_k)  v - \left[ 
    \frac{\nabla f_i(w_k+\delta_k^i v)-\nabla f_i(w_k-\delta_k^i v)}{2\delta_k^i}\right]	\right \| \leq \rho \delta_k^i ~ \|v\|^2 
\end{equation*}
by using the fact the Hessians are $\rho$-Lipschitz continuous. Now if we set 
$$v = \tilde{\nabla} f_i(w_k\!-\!\al \tilde{\nabla} f_i(w_k,\D_{in}^i),\D_{o}^i),$$
then by the definition of $ \tilde{e}_{k}^i $ in \eqref{kln} we can show that 
\begin{equation}
\| \tilde{e}_{k}^i \| \leq \rho ~ \delta_k^i ~ \| \tilde{\nabla} f_i(w_k\!-\!\al \tilde{\nabla} f_i(w_k,\D_{in}^i),\D_{o}^i) \|^2	
\end{equation}
If we replace $\delta_k^i$ by its definition 
\begin{equation}\label{delta_def}
\delta_k^i = \frac{\delta}{\| \tilde{\nabla} f_i(w_k\!-\!\al \tilde{\nabla} f_i(w_k,\D_{in}^i),\D_{o}^i) \|},
\end{equation}
where $\delta:= \frac{1}{6 \rho \alpha}$, then we can show that 
\begin{align}\label{e_{i,k}_i_bound}
\| \tilde{e}_{k}^i \| & \leq  \rho \delta ~ \| \tilde{\nabla} f_i(w_k\!-\!\al \tilde{\nabla} f_i(w_k,\D_{in}^i),\D_{o}^i) \|
\end{align}
Therefore, we have
\begin{align}
\| \E [\tilde{e}_{k}^i] \| &= \left \|\E_{i \sim p} \left [ \E_{\D_o^i, \D_{in}^i} [\tilde{e}_{k}^i ] \right ] \right \| \leq 	\E \left [\E_{\D_o^i, \D_{in}^i} [\|\tilde{e}_{k}^i\|] \right ] \nonumber \\
& \leq \rho \delta ~\E_{i \sim p} \left [ \E_{\D_o^i, \D_{in}^i} \left [\| \tilde{\nabla} f_i(w_k\!-\!\al \tilde{\nabla} f_i(w_k,\D_{in}^i),\D_{o}^i) \| \right ] \right ] \nonumber \\
& \leq  \rho \delta ~ \E_{i \sim p} \left [ \nabla f_i\!\left(w_k-\al \nabla f_i(w_k) \right) \right ] + \rho \delta ~ \frac{\al L \tsigma}{\sqrt{D_{in}}} \label{HF_MAML_eq4} \\
& \leq  \rho \delta ~ \E_{i \sim p} \left [ \left ( I - \al \nabla^2 f_i(w_k) \right )^{-1} {\nabla F_i}(w_k) \right ] + \rho \delta ~ \frac{\al L \tsigma}{\sqrt{D_{in}}} \nonumber \\
& \leq  \frac{\rho \delta}{1- \alpha L} ~ {\nabla F}(w_k) + \rho \delta ~ \frac{\al L \tsigma}{\sqrt{D_{in}}} \label{HF_MAML_eq5}
\end{align}
where \eqref{HF_MAML_eq4} follows from Lemma \ref{lemma:moments}. Considering the bounds \eqref{HF_MAML_eq3.5} and \eqref{HF_MAML_eq5} as well as the result in  \eqref{HF_MAML_eq30}, we can write
\begin{align}\label{HF_MAML_eq6}
\E[G_i(w)] = \nabla F(w_k) + s_k	
\end{align}
with
\begin{align}
\|s_k \| & \leq \frac{\rho \delta \alpha}{1- \alpha L} ~ {\nabla F}(w_k) + (1+ \alpha L + \rho \delta \alpha) ~ \frac{\al L \tsigma}{\sqrt{D_{in}}} \nonumber \\
& \leq 	0.2 {\nabla F}(w_k) + 0.3 \frac{\tsigma}{\sqrt{D_{in}}}.\label{HF_MAML_eq6.25}
\end{align}
where the last inequality is derived using $\alpha L , \rho \delta \alpha \leq 1/6$. As a consequence,  we also have
\begin{equation}\label{HF_MAML_eq6.5}
\| \E[G_i(w)] \|^2 \leq2\|{\nabla F}(w_k)\|^2 + 2\|s_k\|^2 \leq {{2.2}} \|{\nabla F}(w_k)\|^2 + 0.4 \frac{\tsigma^2}{D_{in}}. 	
\end{equation}
Next, to bound the second moment of $G_i(w_k)$, note that
\begin{align}
\E[\|G_i(w_k)\|^2] &\leq 2 \E \left [ \| \tilde{\nabla} f_i\!\left(w_k-\al \tilde{\nabla} f_i(w_k,\D_{in}^i),\D_{o}^i\right) \|^2 \right ] + 2 \alpha^2 \E[\|d_k^i\|^2] \nonumber \\
& \leq 2 \E \left [ \| \tilde{\nabla} f_i\!\left(w_k-\al \tilde{\nabla} f_i(w_k,\D_{in}^i),\D_{o}^i\right) \|^2 \right ] + 2 \alpha^2 \E_{i \sim p} \left [ \E_{\D_o^i, \D_{in}^i, \D_h^i}[\|d_k^i\|^2] \right ] \label{HF_MAML_eq7}
\end{align}
where $\E_{\D_o^i, \D_{in}^i, \D_h^i}[\|d_k^i\|^2]$ can be bounded as follows
{\small
\begin{align}
& \E_{\D_o^i, \D_{in}^i, \D_h^i} [ \|d_k^i\|^2 ] \nonumber \\
& \leq \E_{\D_o^i, \D_{in}^i} \left [ \left \| \frac{ {\nabla} f_i \!\left(w_k\!+\!\delta_k^i \tilde{\nabla} f_i(w_k\!-\!\al \tilde{\nabla} f_i(w_k,\D_{in}^i),\D_{o}^i) \right)\!-\!{\nabla} f_i\!\left(w_k\!-\!\delta_k^i \tilde{\nabla} f_i(w_k\!-\!\al \tilde{\nabla} f_i(w_k,\D_{in}^i),\D_{o}^i) \right) }{2\delta_k^i} \right \|^2 \right. \nonumber \\
& \left. \quad  + \frac{\tsigma^2}{D_h(\delta_k^i)^2} \right ] \nonumber
\end{align}}
and the last inequality comes from the fact that $\text{Var}(X+Y) \leq 2 (\text{Var}(X) + \text{Var}(Y)).$ Moreover, according to the definition in \eqref{kln} we can write that 
\begin{align*}
& \frac{{\nabla} f_i \!\left(w_k\!+\!\delta_k^i \tilde{\nabla} f_i(w_k\!-\!\al \tilde{\nabla} f_i(w_k,\D_{in}^i),\D_{o}^i) \right)\!-\!{\nabla} f_i\!\left(w_k\!-\!\delta_k^i \tilde{\nabla} f_i(w_k\!-\!\al \tilde{\nabla} f_i(w_k,\D_{in}^i),\D_{o}^i) \right)}{2\delta_k^i} \\
& \quad = \nabla^2 f_i(w_k) \tilde{\nabla} f_i(w_k\!-\!\al \tilde{\nabla} f_i(w_k,\D_{in}^i),\D_{o}^i) + \tilde{e}_{k}^i	
\end{align*}
which implies that 
\begin{align}
 \E_{\D_o^i, \D_{in}^i, \D_h^i} [ \|d_k^i\|^2 ] \leq \E_{\D_o^i, \D_{in}^i} \left [ \left \| \nabla^2 f_i(w_k) \tilde{\nabla} f_i(w_k\!-\!\al \tilde{\nabla} f_i(w_k,\D_{in}^i),\D_{o}^i) + \tilde{e}_{k}^i \right \|^2 + \frac{\tsigma^2}{D_h(\delta_k^i)^2} \right ] 
\end{align}
Now, replace $\delta_k^i$ in the second term by its definition in \eqref{delta_def} to obtain
\begin{align}
& \E_{\D_o^i, \D_{in}^i, \D_h^i} [ \|d_k^i\|^2 ] \nonumber \\
& \leq \E_{\D_o^i, \D_{in}^i} \left [ \left \| \nabla^2 f_i(w_k) \tilde{\nabla} f_i(w_k-\al \tilde{\nabla} f_i(w_k,\D_{in}^i),\D_{o}^i) + \tilde{e}_{k}^i \right \|^2 \right]+
\nonumber\\
&\qquad  \E_{\D_o^i, \D_{in}^i} \left[ \frac{\tsigma^2}{D_h \delta^2} \| \tilde{\nabla} f_i(w_k-\al \tilde{\nabla} f_i(w_k,\D_{in}^i),\D_{o}^i) \| ^2 \right ] 
\end{align}
Using this bound along with the inequalities $(a+b)^2 \leq 2 a^2 + 2 b^2$ and 
\begin{equation*}
\left \| \nabla^2 f_i(w_k) \tilde{\nabla} f_i(w_k\!-\!\al \tilde{\nabla} f_i(w_k,\D_{in}^i),\D_{o}^i) \right \|^2 \leq L^2 	\| \tilde{\nabla} f_i(w_k\!-\!\al \tilde{\nabla} f_i(w_k,\D_{in}^i),\D_{o}^i) \|^2
\end{equation*}
we can show that
\begin{align}
& \E_{\D_o^i, \D_{in}^i, \D_h^i} [ \|d_k^i\|^2 ]\nonumber\\
& \leq \E_{\D_o^i, \D_{in}^i} \left [ (\frac{\tsigma^2}{D_h \delta^2} + 2 L^2) \| \tilde{\nabla} f_i(w_k\!-\!\al \tilde{\nabla} f_i(w_k,\D_{in}^i),\D_{o}^i) \|^2 + 2 \| \tilde{e}_{k}^i \|^2 \right ]  \nonumber \\
& \leq (\frac{\tsigma^2}{D_h \delta^2}  + 2 L^2 + 2 \rho^2 \delta^2 ) \E_{\D_o^i, \D_{in}^i} \left [\| \tilde{\nabla} f_i(w_k\!-\!\al \tilde{\nabla} f_i(w_k,\D_{in}^i),\D_{o}^i) \|^2 \right ]
\end{align}
where the last inequality follows from \eqref{e_{i,k}_i_bound}. Plugging this bound in \eqref{HF_MAML_eq7} leads to
\begin{align}
\E[\|G_i(w_k)\|^2] & \leq (2 + 2 \alpha^2 (\frac{\tsigma^2}{D_h \delta^2}  + 2 L^2 + 2 \rho^2 \delta^2 ) ) \E \left [ \| \tilde{\nabla} f_i\!\left(w_k-\al \tilde{\nabla} f_i(w_k,\D_{in}^i),\D_{o}^i\right) \|^2 \right ] \nonumber \\
& \leq 4.3 ~ \E \left [ \| \tilde{\nabla} f_i\!\left(w_k-\al \tilde{\nabla} f_i(w_k,\D_{in}^i),\D_{o}^i\right) \|^2 \right ] 
\end{align}
where the last inequality is derived using $\alpha L \leq 1/6$ along with $\rho \delta \alpha = 1/6$ and $D_h \geq 36 (\rho \alpha \tsigma)^2$. Now, using Lemma \ref{lemma:moments} with $\phi = 10$, we can write
\begin{align}
\E[\|G_i(w_k)\|^2] & \leq 4.3 \left(1+ \frac{1}{10}\right) \E[\|\nabla f_{i} (w_k - \alpha \nabla f_{i}(w_k)\|^2] + 47.3 \al^2 L^2 \frac{\tsigma^2}{D_{in}} + 4.3 \frac{\tsigma^2}{D_o}	 \label{HF_MAML_eq7.5} \\
& \leq 5 \E[\|\nabla f_{i} (w_k - \alpha \nabla f_{i}(w_k))\|^2] + 5 \tsigma^2 \left(\frac{1}{D_{in}} + \frac{1}{D_o}\right) \label{HF_MAML_eq8} \\
& \leq \frac{5}{(1- \alpha L)^2} \E[\|\nabla F_{i} (w_k)\|^2] + 5 \tsigma^2 \left(\frac{1}{D_{in}} + \frac{1}{D_o}\right) \label{HF_MAML_eq9}
\end{align}
where \eqref{HF_MAML_eq8} is a simplification of \eqref{HF_MAML_eq7.5} using $\alpha L \leq 1/6$ and \eqref{HF_MAML_eq9} comes from the fact that 
\begin{equation*}
\|\nabla f_{i} (w_k - \alpha \nabla f_{i}(w_k))\| = \| (I - \alpha \nabla^2 f_i(w_k))^{-1} \nabla F_i(w_k) \| \leq \frac{1}{1 - \alpha L}	\|\nabla F_i(w_k) \|.
\end{equation*}
Now, using \eqref{bound_F_i_by_F} in Lemma \eqref{lemma_bound_by_F}, we can show that
\begin{equation}\label{HF_MAML_eq10}
\E[\|G_i(w_k)\|^2]	\leq 50 \|\nabla F(x)\|^2 + 18 \sigma^2 + 5 \tsigma^2 \left(\frac{1}{D_{in}} + \frac{1}{D_o}\right).
\end{equation}
Once again, the same argument as the proof of Theorem \ref{Thm_SGD_general}, we obtain
\begin{align}\label{thm_HFMAML_ineq2}
\E & [F(w_{k+1})|\F_k]  \leq  F(w_k) - \|\nabla{F}(w_k)\|^2 \left ( \E[\be_{i,k}|\F_k] - \frac{L_k}{2} \E[\be_{i,k}^2|\F_k] (2.2+ \frac{50}{B})\right) \nonumber \\
& + \E[\be_{i,k}|\F_k] \|\nabla{F}(w_k)\| \|s_k\| +  \frac{L_k}{2} \E[\be_{i,k}^2|\F_k] \left ( \frac{1}{B} \left ( 18 \sigma^2  + 5 \tsigma^2 \left(\frac{1}{D_o} + \frac{1}{D_{in}}\right) \right ) + 0.4 \frac{\tsigma^2}{D_{in}} \right ).
\end{align}
Note that, using \eqref{HF_MAML_eq6.25}, we have 
\begin{equation*}
\|\nabla{F}(w_k)\| \|s_k\| \leq \frac{1}{2} \left (\frac{\|\nabla{F}(w_k)\|^2}{2} +  2 \|s_k\|^2 \right ) \leq 0.4{\|\nabla{F}(w_k)\|^2} +  0.18 \frac{\tsigma^2}{D_{in}}.   
\end{equation*}
Plugging this bound in \eqref{thm_HFMAML_ineq2} implies
\begin{align*}
\E & [F(w_{k+1})|\F_k]  \leq  F(w_k) - \|\nabla{F}(w_k)\|^2 \left ( 0.6 \E[\be_{i,k}|\F_k] - \frac{L_k}{2} \E[\be_{i,k}^2|\F_k] (2.2+ \frac{50}{B})\right) \\
& +  \frac{L_k}{2} \E[\be_{i,k}^2|\F_k] \left ( \frac{1}{B} \left ( 18 \sigma^2  + 5 \tsigma^2 \left(\frac{1}{D_o} + \frac{1}{D_{in}}\right) \right ) + 0.4 \frac{\tsigma^2}{D_{in}} \right ) + 0.18 \E[\be_{i,k}|\F_k]  \frac{\tsigma^2}{D_{in}}.
\end{align*}
Using $\beta_k = \tilde{\beta}(w_k)/25$, and with similar analysis as Theorem \ref{Thm_SGD_general}, we obtain
\begin{align*}
\E[F(w_{k+1})|\F_k] & \leq  F(w_k) - \frac{1}{200 L_k} \|\nabla{F}(w_k)\|^2  + \frac{18 \sigma^2  + 5{\tsigma^2}/{D_o}}{1600 L B}+ \frac{\tsigma^2/D_{in}}{100 L}
\end{align*}
which is again similar to \eqref{general_SGD_ineq10}, and the rest of proof follows same as the way that we derived \eqref{thm1_ineq1}- \eqref{telescope4} in the proof of Theorem \ref{Thm_SGD_general}.
\end{proof}

\bibliographystyle{apalike}
\bibliography{main}

\end{document}